\documentclass[12pt,draftcls,onecolumn]{IEEEtran}
\usepackage{textcomp}
\usepackage{url}
\usepackage{verbatim}
\usepackage{etoolbox}
 \usepackage{amsmath,amssymb}
 \usepackage{graphicx,graphics,color,psfrag}
 \usepackage{cite,balance}
 \allowdisplaybreaks[4]
 \usepackage{algorithm}
 \usepackage{accents}
 \usepackage{amsthm}
 \usepackage{bm}
 \usepackage{algorithmic}
 \usepackage[english]{babel}
 \usepackage{multirow}
 \usepackage{enumerate}
 \usepackage{cases}
 \usepackage{stfloats}
 \usepackage{dsfont}
 \usepackage{color,soul}
 \usepackage{amsfonts}
 \usepackage{cite,graphicx,amsmath,amssymb}
 \usepackage{fancyhdr}
 \usepackage{hhline}
 \usepackage{graphicx,graphics}
 \usepackage{array,color}
 \usepackage{amsmath}
\usepackage{float}
\usepackage{amssymb}
\usepackage{amsmath}
\usepackage{amsthm}
\usepackage{amsfonts}
\usepackage{graphicx}
\usepackage{algorithm}
\usepackage{algorithmic}
\usepackage{epstopdf}
\usepackage{cite}
\usepackage{amsmath,bm}
\usepackage{subfigure}
\usepackage{graphicx}
\usepackage{color}
\usepackage{graphicx}
\usepackage{calc}
\usepackage{diagbox}

\newtheorem{theorem}{\textbf{Theorem}}

\newtheorem{lemma}{\textbf{Lemma}}

\newtheorem{Def}{Definition}

\newtheorem{assumption}{Assumption}

\begin{document}
\bibliographystyle{IEEEtran}

\title{Predictive GAN-powered Multi-Objective Optimization for Hybrid Federated Split Learning  \thanks{B. Yin, Z. Chen and M. Tao are with the School of Electronic Information and Electrical Engineering, Shanghai Jiao Tong University, P. R. China. Email: \{yinbsh, zhiyongchen, mxtao\}@sjtu.edu.cn. (\emph{Corresponding author: Zhiyong Chen, Meixia Tao}).}}

\author{Benshun Yin, Zhiyong Chen and Meixia Tao, \emph{IEEE Fellow}}
\maketitle

\begin{abstract}
As an edge intelligence algorithm for multi-device collaborative training, federated learning (FL) can reduce the communication burden but increase the computing load of wireless devices. In contrast, split learning (SL) can reduce the computing load of devices by using model splitting and assignment, but increase the communication burden to transmit intermediate results. In this paper, to exploit the advantages of FL and SL, we propose a hybrid federated split learning (HFSL) framework in wireless networks, which combines the multi-worker parallel update of FL and flexible splitting of SL. To reduce the computational idleness in model splitting, we design a parallel computing scheme for model splitting without label sharing, and theoretically analyze the influence of the delayed gradient caused by the scheme on the convergence speed. Aiming to obtain the trade-off between the training time and energy consumption, we optimize the splitting decision, the bandwidth and computing resource allocation. The optimization problem is multi-objective, and we thus propose a predictive generative adversarial network (GAN)-powered multi-objective optimization algorithm to obtain the Pareto front of the problem. Experimental results show that the proposed algorithm outperforms others in finding Pareto optimal solutions, and the solutions of the proposed HFSL dominate the solution of FL.
\end{abstract}


\section{Introduction}
With the booming development of Internet of things (IoT), abundant data is produced by IoT devices every day \cite{zhu2020toward}. Based on the large amount of distributed data, edge machine learning algorithms are flourishing to realize intelligent applications in wireless networks \cite{wang2020convergence,zhang2019deep,Chen2019,zhou2019edge,zhu2020toward,zap2019,Deng2020,li2018learning,data_nas}. To enable multiple wireless devices to collaboratively train a machine learning algorithm with their local data, federated learning (FL) is proposed \cite{federated}. Compared with traditional centralized learning that transmits huge amounts of raw data to the cloud server for training, FL can effectively protect the privacy of data without exchanging the local data. Meanwhile, since only the global model is uploaded and downloaded in each round of FL training, FL generally reduce the communication load compared to transmitting raw data. However, the IoT devices, also called workers, are required to perform a local update of the training model with local computing power in FL. This can greatly increase the computation burden of the workers, especially for training deep neural networks with high computational complexity. When the computing power of workers is low, the training time of FL is greatly extended, which affects the practical application of FL. In addition, local updating entirely by the own computing power of workers increases their energy consumption.

Unlike FL, split learning (SL) can split a deep neural network into multiple parts and deliver them to different computational nodes  \cite{gupta2018distributed,vepakomma2018split,HiveMind,thapa2020splitfed}, such as workers and edge servers. Through multi-tier computing, the computation burden and energy consumption of a single node can be significantly reduced. However, the intermediate output or gradient of the neural network needs to be transmitted between different computational nodes in SL, yielding a huge communication burden. Besides, due to the data dependency in the forward and backward propagation of neural network training, the computation of the later node has to wait for the required data to be transmitted from the previous node. These two issues can prolong the training time of SL.

Through the above comparison, we can find that FL and SL are complementary in terms of the communication resource consumption of the learning system and computing power requirement of workers. FL requires few communication resources to transmit global models infrequently, but requires workers with powerful computing power to complete local updates.In contrast, SL does not require workers with strong computing power but requires lots of communication resources to frequently transmit intermediate results of neural networks. Inspired by this, if we reasonably combine FL and SL, a trade-off between the training time of the learning system and the energy consumption of workers can be achieved under limited communication and computing resources.

\subsection{Related Work}
For utilizing the multi-worker parallel update of FL and the low computational requirement for workers of SL to improve training speed, the combination of FL and SL has been considered in many works \cite{thapa2020splitfed,9652119,han2021accelerating,he2020group,tian2022fedbert,park2021federated}. In \cite{thapa2020splitfed}, all the workers split the neural network into the worker-side part and the server-side part. The server updates the parameters of the server-side part in parallel after receiving the intermediate output of the worker-side part from all users, while the parameters of the worker-side part need to be transmitted to the server for global averaging. For large-scale workers in a learning system, workers are divided into multiple groups in \cite{9652119}, where the server-side part of each group and each worker-side perform a global average respectively to obtain the global model. For reducing the frequent communication between workers and the server, \cite{han2021accelerating,he2020group} use an auxiliary light-weight neural network and a loss function to update the parameters of the worker-side part without the participation of the server-side part. FL and SL are combined to train a large natural language processing model in \cite{tian2022fedbert}, where the model can be split and distributed to multiple workers and then updated one by one for sequentially partitioned data. The combination of FL and SL is extended to multi-task learning in \cite{park2021federated} with the shared server-side part and task-specific worker-side part. However, existing works combining FL and SL take the same splitting decision for all workers, without optimizing the splitting decision based on heterogeneous computing and communication resources. 

Some recent works have realized parallel computing between different computational nodes through a reasonable arrangement of communication and computing \cite{pipedream,wang2020geryon,DynaComm,wang2021overlap} to reduce the waiting time in SL caused by the data dependency of neural networks. With the neural network split and assigned to different workers, workers can asynchronously train with multiple minibatches of data at the same time in \cite{pipedream} based on the careful arrangement of multiple computing flows, thereby reducing computational idleness. For accelerating training, the multiple computing flows in \cite{wang2020geryon} are assigned different priorities according to the urgency level of the parameter. A layer-wise communication scheduler is designed in \cite{DynaComm} to make parameter transmission and computation overlap as much as possible. Similarly, the computation and communication of neural network training are partitioned with a greedy algorithm to overlap gradient communication with backward computation and parameter communication with forward computation in \cite{wang2021overlap}. However, these parallel mechanisms are designed for split learning with label sharing \cite{vepakomma2018split}, and cannot be directly applied to the case without label sharing.

To obtain a trade-off solution of a multi-objective problem, a common method is to weight the multi-objective problem into a single-objective problem with variable parameters, which requires prior information about the objective preference. Besides, many classic multi-objective evolutionary algorithms (MOEAs) are proposed, such as multiobjective evolutionary algorithm based on decomposition (MOEA/D) \cite{zhang2007moea} and non-dominated sorting genetic algorithm III (NSGA-III) \cite{NSGAIII}. Moreover, numerous model-based MOEAs are proposed using the decision variable clustering \cite{zhang2016decision}, Gaussian process-based inverse modeling \cite{cheng2015multiobjective}, dominance relationship classification \cite{8281523}, Pareto rank learning \cite{6252865}, etc. In general, the performance of these model-based MOEAs degrades as decision variables increase due to the increase of computational complexity and data requirement \cite{MOGAN}. Therefore, the generative adversarial network (GAN) with strong learning ability is applied to multi-objective optimization for generating potential solutions \cite{MOGAN,wang2021manifold}. Furthermore, multi-objective reinforcement learning algorithms \cite{moppo} are proposed to deal with multi-objective optimization problems with sequential decision making.

\subsection{Contributions and Outline}
Motivated by the above, we propose a hybrid federated split learning (HFSL) framework in wireless networks, and then jointly optimize the model splitting decision and communication-computing resource allocation for the HFSL system to reduce the training time and energy consumption of workers. The main contributions of this paper are summarized as follows:

\begin{itemize}
\item We propose a hybrid federated split learning framework to combine the advantage of FL and SL. This framework allows workers to make different splitting decisions to adjust the communication and local computing burden according to their heterogeneous communication and computing conditions, resulting in better trade-offs between training time and energy consumption. 

\item We design a parallel computing scheme for model splitting without label sharing to reduce the computational idleness of workers. In the proposed scheme, the calculation of other minibatches is inserted into the idleness of calculating a minibatch of data to train two minibatches of data simultaneously, which also causes a delayed gradient in the update. We theoretically analyze the effect of the delayed gradient on the convergence rate. We also theoretically derive the number of global rounds required for the hybrid federated split learning system to achieve the desired performance.

\item We propose a predictive GAN-powered multi-objective optimization algorithm to obtain the set of Pareto-dominating splitting decision and resource allocation. Specifically, we design a method to find dominance pairs from the current solutions, which is a pair of solutions and one of them strictly dominating the other. With the dominance pairs, the discriminator is trained to learn features from the difference between the dominating solution and the dominated solution. Then the generator can be trained with the discriminator to predict solutions that dominates the current dominating solutions. Experimental results show that the Pareto front found by the proposed algorithm outperforms other algorithms.
\end{itemize}

The rest of this paper is organized as follows. The system model and the problem formulation are introduced in Section \ref{sys}. The convergence analysis of the hybrid federated split learning is presented in Section \ref{converge}. The detail of the proposed algorithm for multi-objective optimization is introduced in Sections \ref{alg}. Finally, extensive experimental results are presented in Section \ref{sim}, and conclusions are drawn in Section \ref{con}.

\section{System Model}
\label{sys}
We consider a hybrid federated split learning system as shown in Fig. \ref{model}, which consists of a base station (BS) with an edge server and a set of workers $\mathcal{K}=\{1,2,...,K\}$. Each worker $k\in \mathcal{K}$ has its local dataset $\{(\bm{x}_{k,1},y_{k,1}),...,(\bm{x}_{k,D_k},y_{k,D_k})\}$ with the size $D_k$, where $\bm{x}_{k,1},...,\bm{x}_{k,D_k}$ are the raw data and $y_{k,1},...,y_{k,D_k}$ refer to the corresponding label. The total amount of data for all workers is $D=\sum_{k\in \mathcal{K}}D_k$. By multi-worker collaborative learning, this system aims to find the optimal vector $\bm{w}^{\ast}$ that minimizes the global loss function
\begin{equation}
F(\bm{w})=\frac{\sum_{k\in \mathcal{K}}D_k F_k(\bm{w})}{D}=\frac{\sum_{k\in \mathcal{K}}\sum_{d=1}^{D_k}f(\bm{x}_{k,d},y_{k,d};\bm{w})}{D},
\end{equation}
where $F_k(\cdot)$ is the local loss function of worker $k$. $f(\cdot)$ is the task-specific loss function such as cross-entropy. 
\begin{figure}[t]
\centering
\includegraphics[width=8cm]{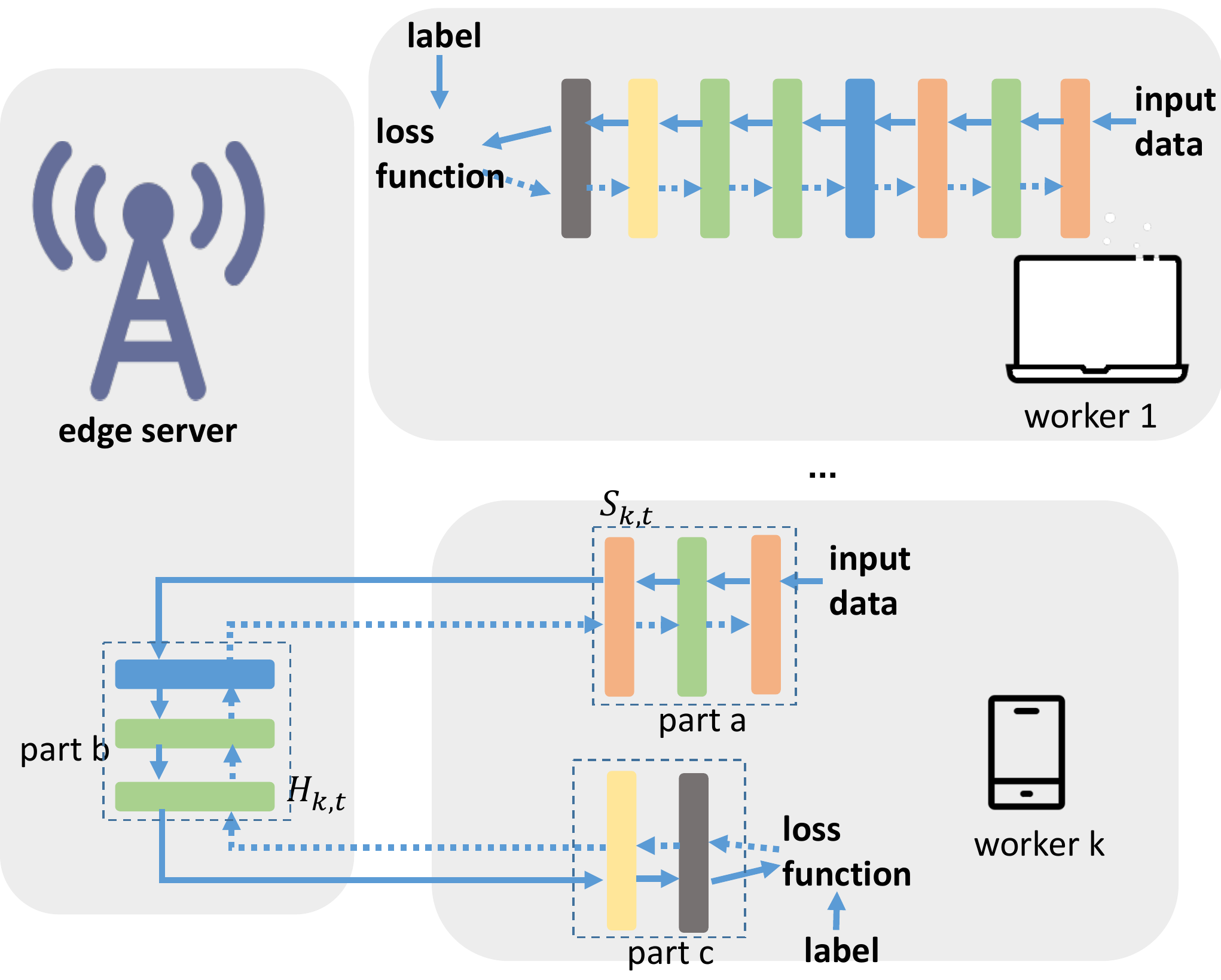}
\caption{A hybrid federated split learning system.}
\label{model}
\end{figure}

\subsection{Hybrid Federated Split Learning}
In the hybrid federated split learning system, if a worker has poor computing power or heavy computing load, the deep neural network trained on the worker can be split to offload part of the computation to the edge server. Specifically, the DNN is split into three parts, namely the input layer to the $S_{k}$-th layer, the ($S_{k}+1$)-th layer to the $H_{k}$-th layer, and the ($H_{k}+1$)-th layer to the output layer, which are denoted as \emph{part a}, \emph{part b} and \emph{part c} respectively. User privacy is protected by keeping the raw data and the corresponding label on the worker, so part a and part c are executed on the worker, while part b can be offloaded to the server. Let $S_{k}$ and $H_{k}$ be the split decision of the worker $k$. In particular, we have $S_{k}=H_{k}$ when the DNN is not split, such as the worker 1 in Fig. \ref{model}. 

In the hybrid federated split learning mechanism, the global aggregation of FL and the device-edge synergy of SL are combined. The execution flow of this mechanism is shown in Fig. \ref{slot}. Step 1, the splitting decision of each worker is made based on computing power and channel conditions before training starts and is fixed in all the global rounds. Step 2, denote the global aggregated model obtained in the last round as $\bm{W}_{t-1}$. The workers download the part of $\bm{W}_{t-1}$ that needs to be executed locally. The whole $\bm{W}_{t-1}$ is downloaded for workers without model splitting, while only the parameter of \emph{part a} and \emph{part c} are downloaded for workers with splitting. Step 3, after the worker $k$ receives the parameter, it performs $N_k$ iterations to locally update the model $\bm{W}_{t-1}$ by 
\begin{equation}
\label{grad}
\bm{w}^n_{k,t}=\bm{w}^{n-1}_{k,t}-\eta \nabla F_k (\bm{w}^{n-1}_{k,t}), n=1,2,...,N_k,
\end{equation}
where $\bm{w}^n_{k,t}$ is the local model obtained by worker $k$ after the $n$-th iteration of the $t$-th global round, and $\bm{w}^0_{k,t}=\bm{W}_{t-1}$. Step 4, after the worker $k$ completes $N_k$ iterations, it uploads the locally executed part of $\bm{w}^{N_k}_{k,t}$ to the server. Step 5, when the edge server receives all the locally updated model, it performs the global aggregation by
\begin{equation}
\label{average}
\bm{W}_{t}=\sum_{k \in \mathcal{K}} \frac{D_k}{D} \bm{w}^{N_k}_{k,t}.
\end{equation}
Steps 2 to 5 above are performed in each global round, and they are repeated for many times to obtain the desired learning performance.

Specifically, in the third step above, the workers without model splitting perform local update of the DNN entirely by their own computing resources, while the workers with splitting require the computing resources of the server. In the forward propagation of model splitting, the output feature of \emph{part a} needs to be uploaded to the server for the calculation of \emph{part b}, and the execution of \emph{part c} depends on the output of \emph{part b}. Similar data dependency exists in the backward propagation because the gradient calculation of the current layer requires the gradient of the previous layer based on the chain rule. Besides, as the backward propagation requires the output value of the hidden layer obtained by the forward propagation, the backward propagation of a minibatch of data can be performed after completing the forward propagation of this minibatch.
\begin{figure*}[t]
\centering
\includegraphics[width=16.5cm]{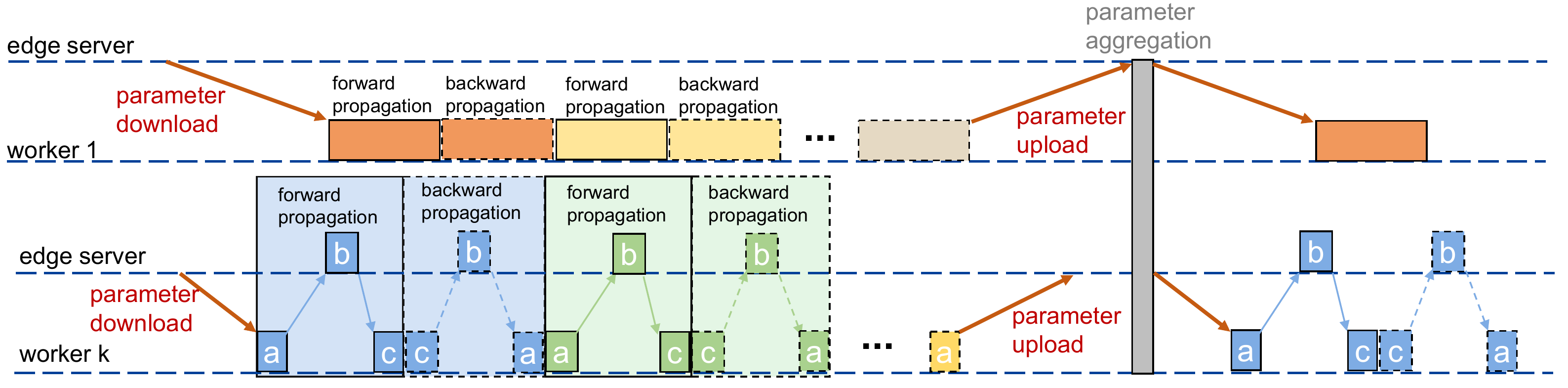}
\caption{Execution flow diagram of the hybrid federated split learning.}
\label{slot}
\end{figure*}

\subsection{Parallel Computing for Model Splitting}
\begin{figure*}[t]
\centering
\subfigure[]{
\includegraphics[width=14cm]{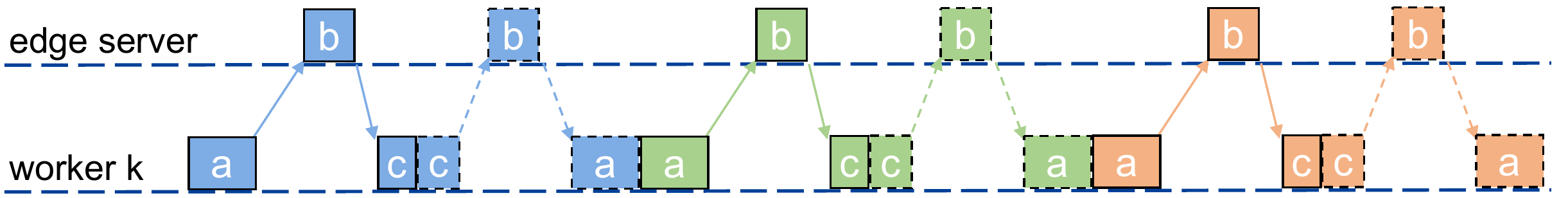}
}
\subfigure[]{
\includegraphics[width=14cm]{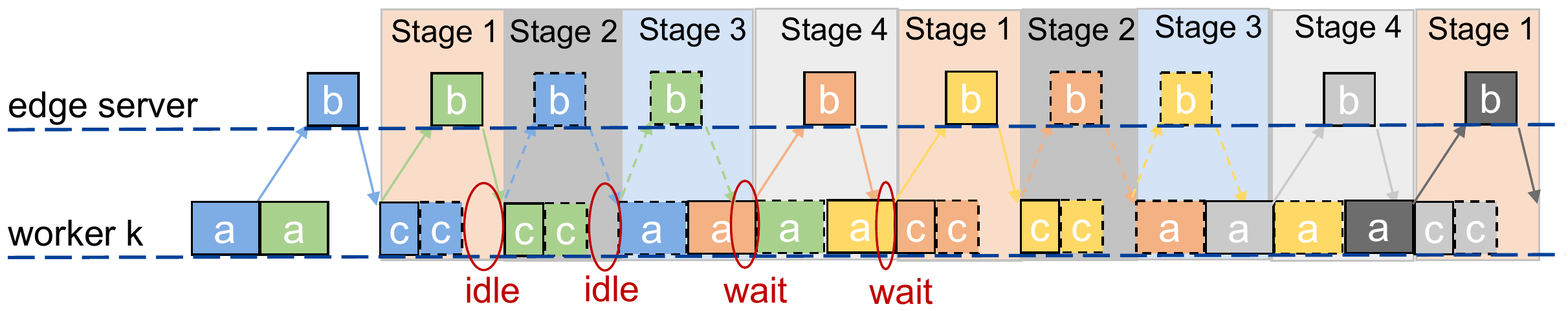}
}
\subfigure[]{
\includegraphics[width=14cm]{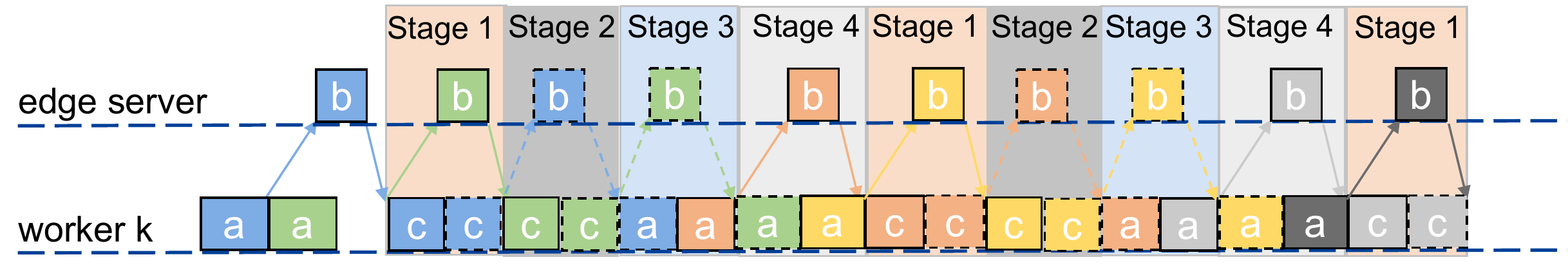}
}
\caption{(a) Execution flow of the vanilla model splitting. (b) Parallel computing process for model splitting. (c) Parallel computing process for model splitting with adaptive local computing frequency.}
\label{parallel}
\end{figure*}
The data dependency in the forward and backward propagation of DNN training leads to the worker idleness in vanilla model splitting as shown in Fig. \ref{parallel}(a). For example, during the period from the completion of \emph{part a} to the start of \emph{part c} in the forward propagation, the worker has to wait for uploading the output of \emph{part a}, the calculation of \emph{part b} by the server and downloading the output of \emph{part b}. This idleness causes the inefficient execution of model splitting. To alleviate this problem, we design a parallel computing mechanism for model splitting as shown in Fig. \ref{parallel} (b). In the vanilla model splitting, the worker cooperates with the server to perform sequential iterations, that is, the next minibatch of data is processed after the forward and backward propagation of the previous minibatch is completed. In the proposed parallel computing schemd, to improve the execution efficiency, when the forward propagation of \emph{part a} of the first minibatch is completed, the worker starts to calculate \emph{part a} of the second minibatch. When the output of \emph{part b} of the first minibatch is downloaded to the worker, the worker uploads the output of \emph{part a} of the second minibatch while starting to calculate the forward and backward propagation of \emph{part c} of the first minibatch. Subsequent iterations follow a similar arrangement, resulting in stage 1 to 4 in Fig. \ref{parallel} (b). In each stage, a forward and a backward propagation of \emph{part a} or \emph{part c} is calculated locally. Each stage starts with the uploading to the server, and ends with the completion of downloading from the server or local computing. By this method, the worker-side calculation of other minibatches is inserted into the idle time of the worker in the vanilla model splitting. 

Moreover, the mismatch between the time consumption of worker computing and edge computing also causes the worker to be idle or waiting. For example, in the stage 1 of Fig. \ref{parallel} (b), the time cost of uploading the output of \emph{part a}, calculating \emph{part b} and downloading the output of \emph{part b} is larger than that of the forward and backward propagation of \emph{part c}. In this case, the CPU frequency of the worker is too high, which leads to idleness and increases the energy consumption of the worker. So the CPU frequency of the worker can be reduced to save unnecessary energy consumption in this case. In another case, such as stage 3 of Fig. \ref{parallel} (b), the CPU frequency of the worker is too low, causing the worker to wait for the completion of the worker-side calculation to enter stage 4 after receiving the gradient from the server. So in this case, the CPU frequency of the worker can be increased to reduce the waiting time. Considering these two cases, we adopt an adaptive local computing frequency as shown in Fig. \ref{parallel} (c). 

Mathematically, let  $C^F_{l}$ and $C^B_{l}$ are the number of floating point operations (FLOPs) required by the $l$-th layer in the forward and backward propagation of processing each data respectively. $L$ is the total number of layers of the trained neural network. Denote the batch size of worker $k$ as $b_k$. The local computing frequency (in cycle/s) of stage 1 is adaptively set by
\begin{equation}
f_{k,t}^{1}=\frac{\sum_{l=H_k+1}^{L} b_k(C^F_{l}+C^B_{l})}{T_{k,t}^{1} n_k},
\end{equation}  
where $n_k$ is the number of FLOPs per cycle. The duration of stage 1 is given by 
\begin{equation}
\label{T_s1}
T_{k,t}^{1}=\max\big\{T_{k,t}^{UF}+T_{k,t}^{EF}+T_{k,t}^{DF} , \frac{\sum_{l=H_k+1}^{L} b_k(C^F_{l}+C^B_{l})}{f_{k}^{\max} n_k}\big\}.
\end{equation}
Here, $f_{k}^{\max}$ is the maximum CPU frequency of worker $k$. $T_{k,t}^{UF}$ and $T_{k,t}^{DF}$ denote the time consumption of uploading the output of \emph{part a} and downloading the output of \emph{part b} in forward propagation respectively, which are calculated by
\begin{align}
T_{k,t}^{UF}= \frac{b_k O^F_{S_k}}{B_k \log (1+\frac{p_k g^2_{k,t}}{B_k N_0})}, \quad T_{k,t}^{DF}= \frac{b_k O^F_{H_k}}{B_k \log (1+\frac{p_0 g^2_{k,t}}{B_k N_0})},
\end{align}
where $O^F_l$ denotes the size (in \emph{bit}) of the intermediate feature output by the $l$-th layer in the forward propagation of each data. $B_k$ is the bandwidth allocated to worker $k$. $p_k$ and $p_0$ refer to the transmit power of worker $k$ and the edge server respectively. $g_{k,t}$ is the channel gain between the worker $k$ to the server, which is assumed to constant in each global round. $N_0$ is the spectral density of the additive white Gaussian noise (AWGN).  In (\ref{T_s1}), $T_{k,t}^{EF}$ is the time cost of the forward propagation of \emph{part b} at the server, which is given by
\begin{equation}
T_{k,t}^{EF}=\frac{\sum_{l=S_k+1}^{H_k} b_k C^F_{l}}{f^E_k n^E},
\end{equation}
where $f^E_k$ is the computing frequency assigned to worker $k$ by the server, and $n^E$ is the number of FLOPs per cycle for the server.

Obviously, worker $k$ calculates with $f_{k}^{\max}$ to reduce wait time in stage 1 when $\frac{\sum_{l=H_k+1}^{L} b_k(C^F_{l}+C^B_{l})}{f_{k}^{\max} n_k}$ $\ge (T_{k,t}^{UF}+T_{k,t}^{EF}+T_{k,t}^{DF})$. Otherwise worker $k$ calculates in a lower frequency to save energy.

Similarly, the local computing frequency in stages 2, 3 and 4 can be obtained by $f_{k,t}^{2}=\frac{\sum_{l=H_k+1}^{L} b_k(C^F_{l}+C^B_{l})}{T_{k,t}^{2} n_k}$, $f_{k,t}^{3}=\frac{\sum_{l=1}^{S_k} b_k(C^F_{l}+C^B_{l})}{T_{k,t}^{3} n_k}$ and $f_{k,t}^{4}=\frac{\sum_{l=1}^{S_k} b_k(C^F_{l}+C^B_{l})}{T_{k,t}^{4} n_k}$, respectively. The duration of stages 2, 3 and 4 are $T_{k,t}^{2}=\max\{T_{k,t}^{UB}+T_{k,t}^{EB}+T_{k,t}^{DB} , \frac{\sum_{l=H_k+1}^{L} b_k(C^F_{l}+C^B_{l})}{f_{k}^{\max} n_k}\}$, $T_{k,t}^{3}=\max\{T_{k,t}^{UB}+T_{k,t}^{EB}+T_{k,t}^{DB}, \frac{\sum_{l=1}^{S_k} b_k(C^F_{l}+C^B_{l})}{f_{k}^{\max} n_k}\}$ and $T_{k,t}^{4}=\max\{T_{k,t}^{UF}+T_{k,t}^{EF}+T_{k,t}^{DF} , \frac{\sum_{l=1}^{S_k} b_k(C^F_{l}+C^B_{l})}{f_{k}^{\max} n_k}\}$, respectively. The time consumption of uploading the intermediate gradient output by \emph{part c} and downloading the gradient output by \emph{part b} in backward propagation are $T_{k,t}^{UB}=\frac{b_k O^B_{H_k+1}}{B_k \log (1+\frac{p_k g^2_{k,t}}{B_k N_0})}$ and $T_{k,t}^{DB}=\frac{b_k O^B_{S_k+1}}{B_k \log (1+\frac{p_0 g^2_{k,t}}{B_k N_0})}$, respectively. Here, $O^B_l$ denotes the size (in \emph{bit}) of the intermediate gradient output by the $l$-th layer in the backward propagation of each data. The time cost of the backward propagation of \emph{part b} at the server is $T_{k,t}^{EB}=\frac{\sum_{l=S_k+1}^{H_k} b_k C^B_{l}}{f^E_k n^E}$.

In the proposed parallel computing scheme, the calculation of the second minibatch starts before the finish of the first minibatch. Generally, the calculation of the $n$-th minibatch starts when the $(n-2)$-th minibatch is completed, without waiting for the $(n-1)$-th minibatch to complete. Therefore, the delayed gradient is used to update parameters in the parallel computing 
\begin{equation}
\label{delay_grad}
\hat{\bm{w}}^n_{k,t}=\hat{\bm{w}}^{n-1}_{k,t}-\eta \nabla F_k (\hat{\bm{w}}^{n-2}_{k,t}), n=1,2,...,N_k,
\end{equation}
where $\hat{\bm{w}}^{-1}_{k,t}=\hat{\bm{w}}^{0}_{k,t}=\bm{W}_{t-1}$. For example, the iteration of the first and the second minibatch in Fig. \ref{parallel} (c), i.e., the blue and the green flow, are based on the parameter $\hat{\bm{w}}^{-1}_{k,t}$ and $\hat{\bm{w}}^{0}_{k,t}$ respectively. When the first minibatch is completed, the third minibatch starts with the parameter $\hat{\bm{w}}^{1}_{k,t}=\hat{\bm{w}}^{0}_{k,t}-\eta \nabla F_k (\hat{\bm{w}}^{-1}_{k,t})$.

\subsection{Time and Energy Consumption}
Denote the parameter size (in \emph{bit}) of the $l$-th layer as $G_l$. The time consumption of downloading and uploading the parameters that need to be updated locally are respectively given by
\begin{align}
T_{k,t}^{ParD}= \frac{\sum_{l=1}^{S_k}G_l+\sum_{l=H_k+1}^L G_l}{B_k \log (1+\frac{p_0 g^2_{k,t}}{B_k N_0})}, \quad
T_{k,t}^{ParU}= \frac{\sum_{l=1}^{S_k}G_l+\sum_{l=H_k+1}^L G_l}{B_k \log (1+\frac{p_k g^2_{k,t}}{B_k N_0})}.
\end{align}

The number of local iterations is $N_k=\lceil \frac{e_k D_k}{b_k} \rceil$, where $e_k$ is the number of local training epochs and $\lceil \cdot \rceil$ is the ceiling function. Supposing that $N_k$ is even\footnote{We can adjust $e_k$ and $b_k$ to make $N_k$ be even.}, the stage 1 and stage 2 are repeated $N_k/2$ times in the parallel computing of model splitting, while stage 3 and stage 4 are repeated $N_k/2-1$ times. Since there are usually dozens of iterations in each global round, the time consumption of local iteration in model splitting is mainly in the repeated stages. For simplicity, we approximate the calculation of entering and exiting the repeated stages as stage 4 and stage 3, respectively. So when worker $k$ iterates with model splitting in the $t$-th global round, its time consumption can be approximated by
\begin{equation}
T_{k,t}^{sp}=\frac{N_k}{2} \sum_{s=1}^4 T_{k,t}^s +T_{k,t}^{ParD}+T_{k,t}^{ParU}.
\end{equation}
Its energy consumption can be approximated as follows \cite{burd1996processor}
\begin{equation}
E_{k,t}^{sp}=\frac{N_k}{2} \big (\sum_{s=1}^4 \epsilon_k(f_{k,t}^s)^3 T_{k,t}^s +2p_k(T_{k,t}^{UF}+T_{k,t}^{UB}) \big )+ p_k T_{k,t}^{ParU},
\end{equation}
where $\epsilon_k$ is the effective capacitance coefficient of worker $k$'s computing chip.

For the workers to iterate locally without model splitting, their local computing frequency is adaptively set to match the time consumption of other workers, thereby reducing waiting time and wasted energy. Specifically, the local computing frequency of workers without splitting is
\begin{equation}
\label{f_nsp}
f_{k,t}^{nsp}=\frac{e_k D_k\sum_{l=1}^{L} (C^F_{l}+C^B_{l})}{(T_{t}^{\max}-T_{k,t}^{ParD}-T_{k,t}^{ParU}) n_k},
\end{equation}
where $T_{t}^{\max}$ is the time consumption of the $t$-th global round, that is, the maximum time required for the workers to complete parameter download, local update and parameter upload. $T_{t}^{\max}$ is
\begin{align}
\label{T_max}
T_{t}^{\max}=\max \{ (1-I_k) T_{k,t}^{sp} +I_k\big (&\frac{e_k D_k\sum_{l=1}^{L} (C^F_{l}+C^B_{l})}{f_k^{\max} n_k}+T_{k,t}^{ParD} 
+T_{k,t}^{ParU} \big )\}_{k \in \mathcal{K}},
\end{align}
where $I_k$ is the indicator for model splitting, which is defined as
\begin{equation}
I_k=\left\{
\begin{aligned}
1, ~~& if~ S_k=H_k \\
0, ~~& if~ S_k \neq H_k. 
\end{aligned}
\right.
\end{equation}
It can be seen from (\ref{T_max}) that $T_{t}^{\max}$ is either the maximum time cost among workers without splitting at the maximum local frequency or the time cost of the slowest worker with splitting. 

Based on the local computing frequency set by (\ref{f_nsp}), the energy consumption of workers without splitting can be given by
\begin{equation}
E_{k,t}^{nsp}= \epsilon_k(f_{k,t}^{nsp})^3 (T_{t}^{\max}-T_{k,t}^{ParD}-T_{k,t}^{ParU})+ p_k T_{k,t}^{ParU}.
\end{equation}
The total energy consumption of all workers in $t$-th global round is given by
\begin{equation}
E_t^{sum}=\sum_{k \in \mathcal{K}} \big ( (1-I_k) E_{k,t}^{sp} +I_k E_{k,t}^{nsp}\big ).
\end{equation}

\subsection{Problem Formulation}
By optimizing the model splitting decision and allocation of bandwidth and server computing resource, the time consumption of the whole training process and the energy consumption of the workers can minimized. Denote $\bm{\varphi}\triangleq [S_1,H_1,f_1^E, B_1,...,S_K,H_K,f_K^E, B_K]$ be the optimization variables. Therefore, the optimization problem is formulated as
\begin{align}
\min_{\bm{\varphi}}~~ & V_1(\bm{\varphi})=\sum_{t=1}^{\tau} T_t^{\max} \nonumber \\
\min_{\bm{\varphi}}~~ & V_2(\bm{\varphi})=\sum_{t=1}^{\tau} E_t^{sum} \nonumber \\
s.t. ~~~~ & \sum_{k\in \mathcal{K}}(1-I_k)f_k^E \le f^{E,\max}, \label{compute}\\
& \sum_{k\in \mathcal{K}}B_k \le B^{\max}, \label{band}\\
& 1 \leq S_k \le H_k < L, \forall k \in \mathcal{K}, S_k, H_k \in \mathbb{Z}, \label{decision}
\end{align}
where $\tau$ is the number of global rounds required to achieve the desired performance.  For convenience, we denote $\bm{V}(\bm{\varphi})=[V_1(\bm{\varphi}),V_2(\bm{\varphi})]$. The constraints (\ref{compute}) and (\ref{band}) denote the feasible regions of allocated server computing frequency and bandwidth, respectively. The constraint (\ref{decision}) indicates that the input and output layer should be kept on the worker to protect privacy. Besides, the decision $S_k$ and $H_k$ are integers. 

We consider the optimization variables $\bm{\varphi}$ are the same in each global round. This is because optimizing these variables in each round requires full knowledge of the channel conditions in each round, which is difficult in the practical system. In this paper, we consider that the server has the large-scale fading coefficients of all workers and the channel gain changes independently and identically (i.i.d.) over rounds. 

The optimization problem is multi-objective, which is based on the trade-off between time and energy consumption. Increasing the number of workers with model splitting can reduce the energy consumption of the splitting workers, but it can also reduce the computing power allocated by the server to each worker and increase the training time. Thus, we can try to find as many Pareto optimal solutions as possible, instead of a single optimal policy. Meanwhile, the optimization problem is non-convex, and it is generally hard to obtain the set of Pareto-dominating solutions for such a problem. Therefore, a predictive GAN powered multi-objective optimization algorithm is proposed to solve the formulated problem in the following sections.

\section{Convergence Analysis of the proposed HFSL}
\label{converge}
In this section, we first analyze the convergence rate of the proposed HFSL because the delayed gradient update may affect the convergence rate. Although the existing convergence analyses \cite{zhu2021delayed,2014delayed} consider the delayed gradient, the cause of their delayed gradient is different from this paper. In \cite{zhu2021delayed}, workers start the next round of local updates at the same time as global averaging, causing a delay between the arrival of the averaged gradient to workers and the local update. In \cite{2014delayed}, gradients uploaded by workers to the server can be delayed due to asynchronous parameter aggregation. In this paper, the delayed gradient is due to the parallelism between different minibatches of data. Another difference is that we consider the heterogeneous workers, and the delayed gradient only exists on the workers with splitting. We then theoretically derive the number of global rounds required to reach the desired performance in this section.

\subsection{Assumptions}
\begin{assumption}
($L$-smoothness) The local loss function $F_k (\cdot)$ is $L$-smoothness with $L>0$, i.e.,
\begin{equation}
F_k (\bm{w}_2)-F_k (\bm{w}_1) \le  \left\langle \nabla F_k (\bm{w}_1) , \bm{w}_2-\bm{w}_1 \right\rangle +\frac{L}{2} \| \bm{w}_2-\bm{w}_1 \|^2, \forall \bm{w}_1, \bm{w}_2.
\end{equation}
\end{assumption}

\begin{assumption}
($\mu$-strongly convex) The global loss function $F (\cdot)$ is $\mu$-strongly convex with $\mu>0$, i.e., 
\begin{equation}
F (\bm{w}_2) \geq F (\bm{w}_1) + \left\langle \nabla F (\bm{w}_1) , \bm{w}_2-\bm{w}_1 \right\rangle +\frac{\mu}{2} \| \bm{w}_2-\bm{w}_1 \|^2, \forall \bm{w}_1, \bm{w}_2.
\label{convex}
\end{equation}
\end{assumption}

\begin{assumption}
(Bounded gradients) There exits a constant bound $G>0$ on the second moment of the gradients:
\begin{equation}
\mathbb{E} \| F_k (\bm{w}) \|^2 \leq G^2, ~~\forall \bm{w},\forall k.
\end{equation}
\end{assumption}

\begin{assumption}
(Global optimal) There exits a vector $\bm{w}^{\ast}$ that minimizes the global loss function, i.e., $\nabla F(\bm{w}^{\ast})=0$.
\end{assumption}

\subsection{Analysis of Convergence Bound}
Using the parallel computing mechanism, workers with splitting iterate with (\ref{delay_grad}), while workers without splitting use (\ref{grad}). So we modify the global aggregation (\ref{average}) as follows 
\begin{equation}
\bm{W}_{t}=\sum_{k \in \mathcal{K}} \frac{D_k}{D} \big ( I_k \bm{w}^{N_k}_{k,t} + (1-I_k) \hat{\bm{w}}^{N_k}_{k,t} \big ).
\end{equation}
We define the average parameter of all workers at the $n$-th iteration of the $t$-th global round as 
\begin{equation}
\overline{\bm{w}}^{n}_{t}=\sum_{k \in \mathcal{K}} \frac{D_k}{D} \big ( I_k \bm{w}^{n}_{k,t} + (1-I_k) \hat{\bm{w}}^{n}_{k,t} \big ), 
\end{equation}
where $n=1,2,...,N^{\max}$, and ${N^{\max}}=\max\{N_k\}_{k\in \mathcal{K}}$ is the maximum number of iterations among workers. To avoid confusion, we set $\hat{\bm{w}}^{n}_{k,t}=\hat{\bm{w}}^{N_k}_{k,t}$ and $\bm{w}^{n}_{k,t}=\bm{w}^{N_k}_{k,t}$ when $n\geq N_k$. Then we have a key lemma about $\overline{\bm{w}}^{n}_{t}$ as follows. 
\begin{lemma}
The difference between the parameter of worker $k$ and the average parameter across all workers is bounded as follows
\begin{align}
\mathbb{E}\big[\| \overline{\bm{w}}_t^n-\bm{w}^n_{k,t}\|^2 \big ] &\leq 4\eta^2 n^2 G^2, \\
\mathbb{E}\big[\| \overline{\bm{w}}_t^n-\hat{\bm{w}}^{n-1}_{k,t} \|^2 \big ] &\leq 2\eta^2 (n^2 +(n-1)^2 )G^2.
\end{align}
\end{lemma}
\begin{proof}
Please refer to Appendix \ref{lemma1}.
\end{proof}

\begin{theorem}
If the learning rate $\eta \leq \frac{1}{L}$, the global model $\bm{W}_{\tau}$ satisfies the following inequality after $\tau$ global rounds
\begin{align}
\mathbb{E}[F(\bm{W}_{\tau})] - F(\bm{w}^{\ast})& \leq  \rho^{{N^{\max}}\tau} \mathbb{E}\big [F(\bm{W}_{0})- F(\bm{w}^{\ast}) \big ] +\hat{\alpha},
\end{align}
where $\rho=1-\mu \eta \in (0,1)$ is the convergence rate. $\hat{\alpha}=\sum_{t=0}^{\tau-1} (1-\mu \eta)^{N^{\max}t} \sum_{n=0}^{{N^{\max}}-1}(1-\mu \eta)^{n} \alpha ({N^{\max}}-n)$ and $\alpha(n)= \eta^3 G^2 L^2 \sum_{k=1}^K \frac{D_k}{D}\big ( 2 I_k (n-1)^2 +(1-I_k)((n-1)^2+(n-2)^2) \big )$.
\end{theorem}
\begin{proof}
Please refer to Appendix \ref{theorem1}.
\end{proof}

\subsection{Discussion}
Theorem 1 shows that the convergence rate $\rho$ is independent of the indicator $I_k$, which indicates that the delayed gradient of parallel computing does not affect the convergence rate. 

Supposing that the convergence condition of the global model is that $\mathbb{E}[F(\bm{W}_{\tau})] - F(\bm{w}^{\ast})$ is less than $\varepsilon$. We can theoretically derive the number of global rounds required for convergence by making the upper bound in Theorem 1 less than $\varepsilon$, i.e.,
\begin{equation}
\rho^{{N^{\max}}\tau} \mathbb{E}\big [F(\bm{W}_{0})- F(\bm{w}^{\ast}) \big ] +\hat{\alpha} \le \varepsilon.
\end{equation}
Denote $\phi=\frac{\varepsilon-\hat{\alpha}}{\mathbb{E} [F(\bm{W}_{0})- F(\bm{w}^{\ast}) ]}$, then we can deduce that
\begin{equation}
\tau \geq \frac{\log_{\rho}\phi}{N^{\max}}.
\end{equation}

\section{Predictive GAN-powered Multi-objective Optimization Algorithm}
\label{alg}

\begin{figure*}[t]
\centering
\includegraphics[width=16cm]{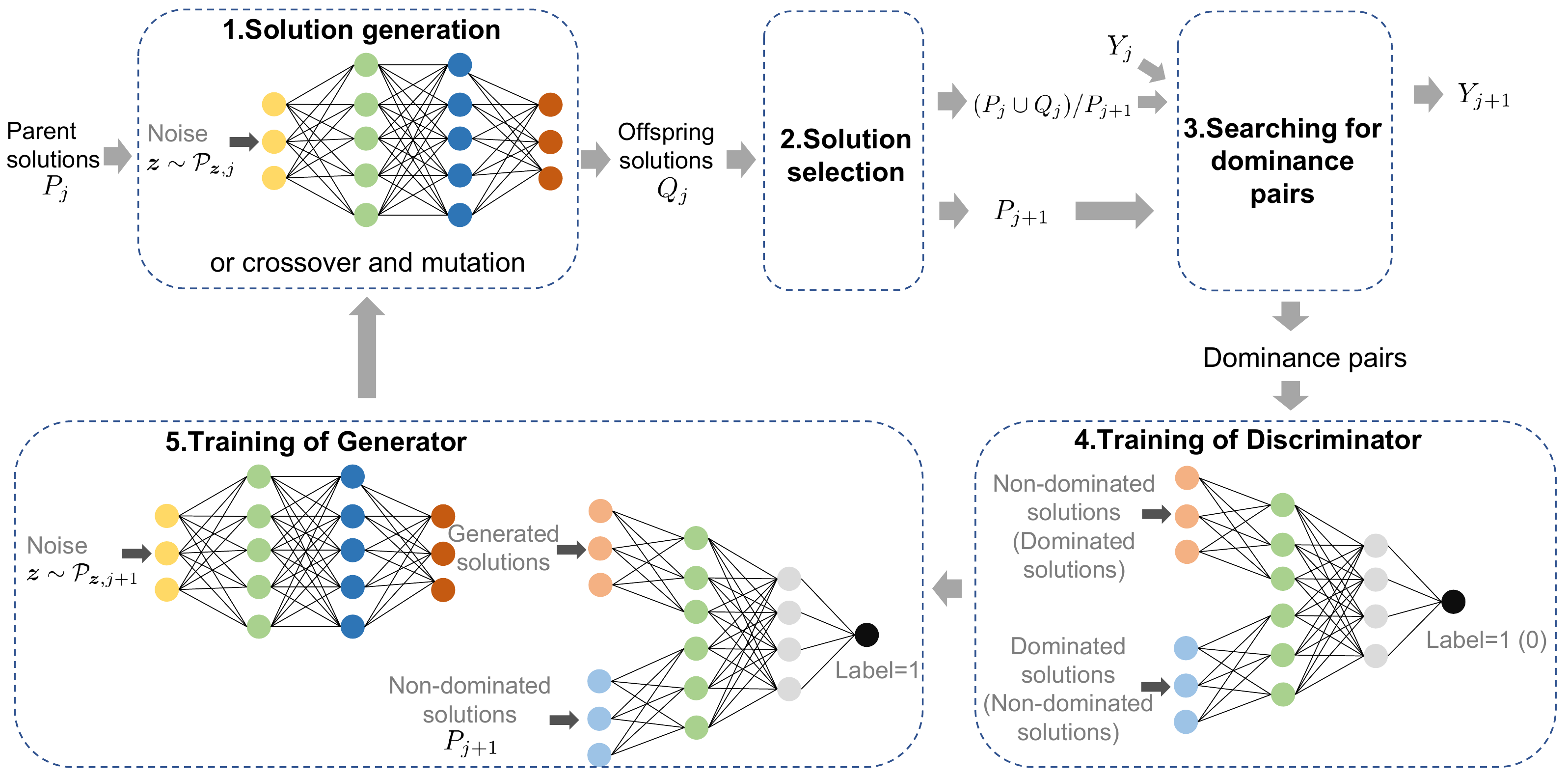}
\caption{Framework of predictive GAN-powered multi-objective optimization algorithm.}
\label{predGAN}
\end{figure*}
To obtain the Pareto front of the optimization problem, a predictive GAN-powered multi-objective optimization algorithm is proposed in this section.
\subsection{Preliminary}
\subsubsection{Generative Adversarial Network}
GAN \cite{GAN} is a powerful generative model capable of synthesizing realistic data, such as images and text. The basic architecture of GAN consists of a generator network and a discriminator network, which play a zero-sum game during training. With the generator $\mathcal{G}$, the input random noise vector $\bm{z}$ sampled from the distribution $\mathcal{P}_{\bm{z}}$ is mapped to the generated data distribution $\mathcal{P}_{\mathcal{G}(\bm{z})}$. The optimization of the generator, i.e., the mapping function, is to make the difference between the generated data distribution $\mathcal{P}_{\mathcal{G}(\bm{z})}$ and the real data distribution $\mathcal{P}_{data}$ as small as possible. Commonly used metrics to measure the difference between different distributions include KL divergence \cite{KLD}, JS divergence \cite{JSD} and Wassertein distance \cite{WaD}. The discriminator $\mathcal{D}$ can predict the probability that the input is real data, and it is used to distinguish the real data and the generated fake data. The optimization goal of the discriminator is to make the output $\mathcal{D}(\mathcal{G}(\bm{z}))$ close to 0 when the input is generated data $\mathcal{G}(\bm{z})$, and close to 1 when the input is real data. Specifically, the generator and discriminator are trained with the following min-max function
\begin{equation}
\min_{\mathcal{G}} \max_{\mathcal{D}} \mathbb{E}_{\bm{x}\sim \mathcal{P}_{data}}[\log \mathcal{D}(\bm{x})]+\mathbb{E}_{\bm{z}\sim \mathcal{P}_{\bm{z}}}[\log (1- \mathcal{D}(\mathcal{G}(\bm{z})))],
\end{equation}
where $\bm{x}\sim \mathcal{P}_{data}$ represents that $\bm{x}$ follows the distribution $\mathcal{P}_{data}$. The generator tries to fool the discriminator by maximizing $\mathcal{D}(\mathcal{G}(\bm{z}))$. After many rounds of training, the discriminator cannot distinguish the real and fake data, which indicates that the data $\mathcal{G}(\bm{z})$ generated by the generator is close to the real data. 

\subsubsection{Solution Selection of NSGA-III}
NSGA-III \cite{NSGAIII} is a classic multi-objective evolutionary algorithm (MOEA) with excellent performance. Similar to many MOEAs, NSGA-III generally has three steps: offspring generation, solution evaluation and solution selection. To generate offspring solutions, genetic operators (i.e., crossover and mutation) are used based on the parent solutions. Then the offspring solutions are evaluated with the optimized function. After evaluation, the parent and offspring solutions are combined to select new parent solutions for the next generation. These steps are repeated for many generations until the termination criterion is satisfied. To select new parent solutions, the definition of Pareto dominance is necessary.
\begin{Def}
(Pareto dominance) For any two candidate solutions $\bm{\varphi}$ and $\bm{\varphi}^{'}$ with the optimized two-objective function $\bm{V}(\bm{\varphi})$, we have
\begin{itemize}
\item $\bm{\varphi} \prec \bm{\varphi}'$ indicates that the solution $\bm{\varphi}^{'}$ is strictly dominated by $\bm{\varphi}$. Specifically, the two function values obtained by $\bm{\varphi}$ are both not greater than that of $\bm{\varphi}^{'}$, i.e., $V_1(\bm{\varphi}) \le V_1(\bm{\varphi}')$ and $V_2(\bm{\varphi}) \le V_2(\bm{\varphi}')$. Meanwhile, at least one of the two function values obtained by $\bm{\varphi}$ is strictly smaller than the corresponding value of $\bm{\varphi}'$.
\item When one element of $\bm{V}(\bm{\varphi})$ is strictly larger than that of $\bm{V}(\bm{\varphi}')$ and the other one is strictly smaller than that of $\bm{V}(\bm{\varphi}')$, the solutions $\bm{\varphi}$ and $\bm{\varphi}'$ are incomparable.
\item If there is no solution $\bm{\varphi}'$ that satisfies $\bm{\varphi}' \prec \bm{\varphi}$, the solution $\bm{\varphi}$ is Pareto optimal. The set of function values achieved by the Pareto optimal solutions is called the Pareto front.
\end{itemize}
\end{Def}

Suppose the set of parent and offspring solutions are denoted as $P_j$ and $Q_j$ at the $j$-th generation, respectively. Let $|P_j|$ denote the size of the set $P_j$, and we have $|P_j|=|Q_j|=R$. Then we need to select $R$ elites from the set $P_j \cup Q_j$ as the parent solution of the next generation, i.e., $P_{j+1}$. Specifically, the non-dominated solutions in $P_j \cup Q_j$ are selected into the first non-domination level $X_1$. In general, the $r$-th non-domination level $X_r (r>1)$ consists of the non-dominated solutions in $(P_j \cup Q_j)/ \cup_{r^{'}=1}^{r-1}X_{r'}$. This selection step is carried out sequentially from $X_1$ to $X_r$, where the size $|\cup_{r^{'}=1}^{r}X_{r'}| \geq R$ is satisfied for the first time. If $|\cup_{r^{'}=1}^{r}X_{r'}|=R$, we have $P_{j+1}=\cup_{r^{'}=1}^{r}X_{r'}$. If $|\cup_{r^{'}=1}^{r}X_{r'}|>R$, the first level to the $(r-1)$-th level are first selected to $P_{j+1}$, and then the remaining $R-|\cup_{r^{'}=1}^{r-1}X_{r'}|$ solutions are chosen from $X_r$. In the remaining selection process of NSGA-III \cite{NSGAIII}, the distance to the reference line is considered to ensure the diversity of obtained solutions. 

\begin{algorithm}[t]
\caption{Predictive GAN-powered Multi-objective Optimization Algorithm}
\label{outline}
\begin{algorithmic}[1]
\STATE Initialize the parameter of GAN
\STATE Initialize $R$ candidate solutions as the parent set $P_1$
\STATE Calculate the mean and variance of $P_1$ as $\bm{\mu}_1$ and $\bm{\sigma}_1$, respectively
\STATE Initialize $R$ candidate solutions as the set being compared $Y_1$

\FOR{$j=1$ to $J$}
\STATE Randomly sample a value $\delta$ from a uniform distribution between 0 and 1
\IF{$\delta\geq 0.5$}
\STATE Generate $R$ offspring solutions as $Q_j$ with genetic operators based on $P_j$
\ELSE
\STATE Generate $R$ offspring solutions as $Q_j$ using the generator with the mean and variance of input noise set as $\bm{\mu}_j$ and $\bm{\sigma}_j$, respectively
\ENDIF
\STATE Select $R$ solutions from $P_j \cup Q_j$ as $P_{j+1}$ using the solution selection of NSGA-III
\STATE Calculate the mean and variance of $P_{j+1}$ as $\bm{\mu}_{j+1}$ and $\bm{\sigma}_{j+1}$, respectively
\STATE Search for dominance pairs between $P_{j+1}$ and $Y_j \cup ((P_j \cup Q_j)/P_{j+1})$, and obtain $Y_{j+1}$ with Algorithm \ref{dominance_pair}
\FOR{$m=1$ to $M$}
\STATE Train the discriminator with the loss function (\ref{dis_loss})
\STATE Generate $R$ solutions using the generator with the mean and variance of input noise set as $\bm{\mu}_{j+1}$ and $\bm{\sigma}_j$, respectively
\STATE Train the generator with the loss function (\ref{gen_loss})
\ENDFOR
\ENDFOR
\end{algorithmic}
{\bf Output:} The Pareto optimal solutions and Pareto front.
\end{algorithm}

\subsection{Overall Framework}


The overall framework of the proposed algorithm is shown in Fig. \ref{predGAN}, and the proposed algorithm is outlined in Algorithm \ref{outline}. Firstly, a hybrid solution generation method is applied to generate the set of offspring solutions $Q_j$ with the parent solutions $P_{j}$. Secondly, the solution selection method of NSGA-III is used to choose $R$ solutions from $P_j \cup Q_j$ to obtain the parent solutions of the next generation $P_{j+1}$. Then dominance pairs between $P_{j+1}$ and $Y_j \cup ((P_j \cup Q_j)/P_{j+1})$ are searched with Algorithm \ref{dominance_pair}. $Y_j$ is a set to keep the worse solutions that are close to the parent solutions. Next, the discriminator is trained to learn the relationship between the dominated solutions and the dominating solutions using the dominance pairs. Finally, the generator is trained with the parameters of discriminator frozen to produce better solutions based on the dominating solutions. These steps are repeated until the termination criterion is satisfied. The details about the dominance pair and the training of GAN are introduced in the following subsections.

\subsection{Dominance Pair and Training of Discriminator}

Dominance pair is a pair of solutions $(\bm{\varphi},\bm{\varphi}')$ in the proposed algorithm, which is used to train the discriminator to learn the dominance relationship between solutions. In the dominance pair, we require that the solution $\bm{\varphi}'$ is strictly dominated by the other. We hope to learn features from the difference between the dominating solution and the dominated solution, thus guiding the generation of offspring solutions. Intuitively, when the distance between the values of the two solutions is very large, the two solutions may be very different, and the guiding effect on the direction may be very small. So we also limit the distance $\|\bm{V}(\bm{\varphi})-\bm{V}(\bm{\varphi}')\|_2$ no more than a threshold $\gamma$. In the proposed algorithm, dominance pairs are search between the $P_{j+1}$ and $Y_j \cup ((P_j \cup Q_j)/P_{j+1})$ using the above two requirements. As shown in Fig. \ref{dom_pair_fig}, let the solution corresponding to the brown point be the dominating solution $\bm{\varphi}$ of a dominance pair, then the value of the other solution is distributed in the fan-shaped area enclosed by the green dotted line. Here, the set of dominated solutions $Y_j \cup ((P_j \cup Q_j)/P_{j+1})$ is composed by half of the solutions in $P_j \cup Q_j$ with poor performance and the set being compared $Y_j$. As the solutions in $(P_j \cup Q_j)/P_{j+1}$ may be far away or very close to the solution in $P_{j+1}$, we add an extra set $Y_j$ to keep some close to solutions, improving the stability of training.

\begin{figure}[t]
\centering
\includegraphics[width=6.5cm]{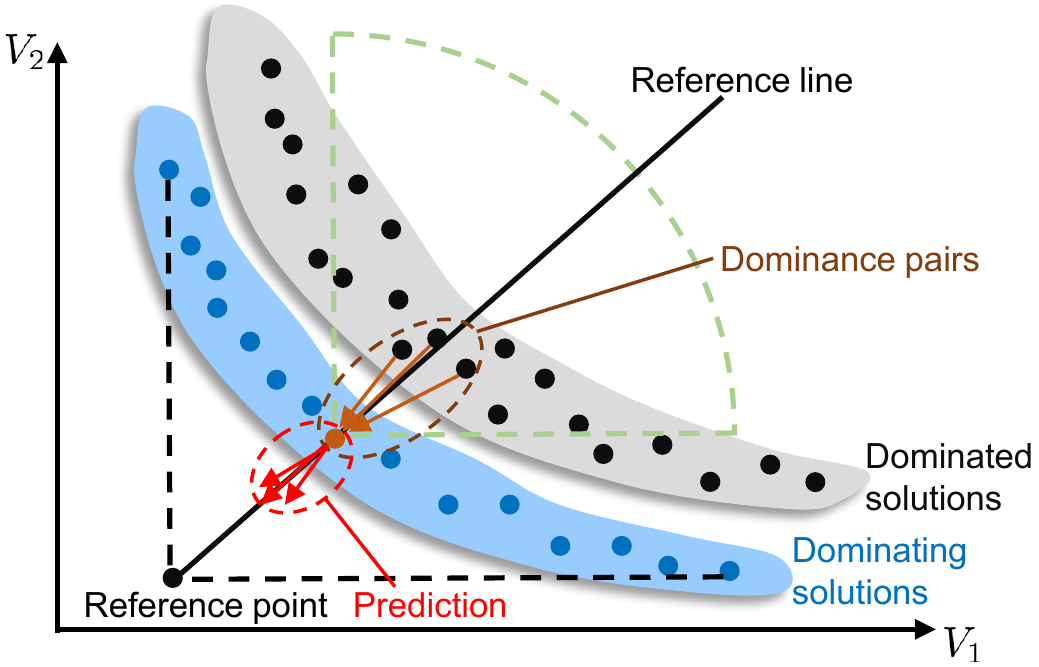}
\caption{Diagram of dominance pair.}
\label{dom_pair_fig}
\end{figure}

We can observe that there are many points in the fan-shaped area of Fig. \ref{dom_pair_fig}, which can form dominance pairs with the brown point. To ease the training burden by reducing training data, we use the reference point to avoid too many dominance pairs formed by one dominating point. Specifically, the coordinates of the reference point are the minimum values of $V_1$ and $V_2$ in all current solutions as shown in Fig. \ref{dom_pair_fig}, which is denoted by $\bm{\psi}$. The reference point is connected to the brown point to form a reference line $\Psi (\bm{\psi},\bm{V}(\bm{\varphi}))$. To reduce the number of dominance pairs corresponding to the brown point, the distance between the reference line and the value of the dominated solutions in the dominance pairs is calculated. The $\kappa$ dominated solutions with the smallest distance are retained. The details of searching and selecting dominance pairs are introduced in Algorithm \ref{dominance_pair}.

After the search and selection of dominance pairs, we can obtain a set of dominance pairs $Z$ as the training data of the discriminator. As shown in Fig. \ref{predGAN}, the discriminator has two inputs in the proposed algorithm, i.e., the dominating solution $\bm{\varphi}$ and the dominated solution $\bm{\varphi}'$ of a dominance pair $(\bm{\varphi},\bm{\varphi}')\in Z$, and they are not commutative. The output of the discriminator represents the probability that one input strictly dominates the other. Specifically, the output $\mathcal{D}(\bm{\varphi} \| \bm{\varphi}')$ is expected to be 1, while $\mathcal{D}(\bm{\varphi}' \| \bm{\varphi})$ is expected to be 0. The discriminator is trained in a supervised method, and the loss function is defined as
\begin{equation}
\label{dis_loss}
F_{\mathcal{D}}=-\frac{1}{|Z|} \sum_{(\bm{\varphi},\bm{\varphi}')\in Z} \big ( \ln(\mathcal{D}(\bm{\varphi} \| \bm{\varphi}')) +\ln (1-\mathcal{D}(\bm{\varphi}' \| \bm{\varphi})) \big ).
\end{equation}

\begin{algorithm}[t]
\caption{Searching for Dominance Pairs}
\label{dominance_pair}
{\bf Input:} The set of dominating solutions $P_{j+1}$ and the set of dominated solutions $Y_j \cup ((P_j \cup Q_j)/P_{j+1})$.
\begin{algorithmic}[1]
\STATE Initialize the set of dominance pairs $Z=\emptyset$
\STATE Initialize the set being compared in the next generation $Y_{j+1}=\emptyset$
\STATE Set the coordinates of the reference point as $\bm{\psi}=[\min_{\bm{\varphi} \in P_{j+1}} V_1(\bm{\varphi}),\min_{\bm{\varphi} \in P_{j+1}} V_2(\bm{\varphi})]$

\FOR{each solution $\bm{\varphi} \in P_{j+1}$}
\STATE Initialize the set of dominance pairs corresponding to $\bm{\varphi}$ as $Z_{\bm{\varphi}}=\emptyset$
\FOR{each solution $\bm{\varphi}' \in Y_j \cup ((P_j \cup Q_j)/P_{j+1})$}
\IF{$(\bm{\varphi} \prec \bm{\varphi}')$ \& $( \|\bm{V}(\bm{\varphi})-\bm{V}(\bm{\varphi}')\|_2 \leq \gamma )$ }
\STATE Add the dominance pair $(\bm{\varphi}, \bm{\varphi}')$ to $Z_{\bm{\varphi}}$
\ENDIF
\ENDFOR
\IF{$|Z_{\bm{\varphi}}|>\kappa$ }
\STATE Calculate the distance between the value of $\bm{\varphi}'$ of the dominance pairs in $Z_{\bm{\varphi}}$ and the reference line $\Psi (\bm{\psi},\bm{V}(\bm{\varphi}))$
\STATE Keep the $\kappa$ pairs with the smallest distance, and remove the rest from $Z_{\bm{\varphi}}$
\ENDIF
\STATE Add the dominated solution $\bm{\varphi}'$ of the dominance pairs in $Z_{\bm{\varphi}}$ to $Y_{j+1}$
\STATE Merge the set $Z_{\bm{\varphi}}$ into $Z$ 
\ENDFOR
\end{algorithmic}
{\bf Output:} The set of dominance pairs $Z$ and the set being compared in the next generation $Y_{j+1}$.
\end{algorithm}

\subsection{Training of Generator and Solution Generation}


With the trained discriminator, we can train the generator to predict solutions with better performance than the current dominating solutions as shown by the red arrow in Fig. \ref{dom_pair_fig}. As shown in Fig. \ref{predGAN}, the input of the generator is a random noise vector $\bm{z}$, which is sampled from a multivariate normal Gaussian distribution $\mathcal{P}_{\bm{z},j}$. The mean vector and covariance matrix of the Gaussian distribution are obtained from the current set of dominating solutions
\begin{align}
\bm{\mu}_j=\frac{1}{|P_j|}\sum_{\bm{\varphi} \in P_j} \bm{\varphi},\quad
\bm{\sigma}_j=\frac{1}{|P_j|}\sum_{\bm{\varphi} \in P_j} (\bm{\varphi}-\bm{\mu}_j)(\bm{\varphi}-\bm{\mu}_j)^T.
\end{align} 
This setting is helpful to generate solutions that approximate the given dominating solutions and reduce the training difficulty \cite{MOGAN}. The output of the generator is the generated solution $\mathcal{G}(\bm{z})$, which is expected to be better than the solutions in $P_j$. The loss function of the generator is
\begin{equation}
F_{\mathcal{G}}=-\frac{1}{|P_j|} \sum_{\varphi \in P_j} \ln \mathcal{D}(\mathcal{G}(\bm{z})\| \varphi ), \bm{z} \sim \mathcal{P}_{\bm{z},j},
\label{gen_loss}
\end{equation}
which is minimized by gradient descent methods. We can see that the probability $\mathcal{D}(\mathcal{G}(\bm{z})\| \varphi )$ is maximized to make the generated solution $\mathcal{G}(\bm{z})$ strictly dominate the solutions in $ P_j$. 

For the solution generation, we adopt the hybrid generation method \cite{MOGAN}. To avoid the mode collapse \cite{WGAN} in the training of GAN, genetic operations and the generator are used to generate offspring solutions with the same probability.

\section{Simulation Results}
\label{sim}
In this section, we present the detail of the simulation and evaluate the performance of the proposed HFSL based on the predictive GAN-powered multi-objective optimization algorithm.

\subsection{Simulation Setup}
\subsubsection{Simulation Environment}
A circular network with a BS at the center is considered in the simulation, serving for $K=16$ workers. The distance $d_k$ between the worker $k$ and the BS is distributed uniformly within 2 to 50 meters. The channel gain $g_{k,t}$ follows the Rayleigh distribution with the mean $10^{-PL(d_k)/20}$, where we consider the path loss $PL(d_k)(dB)=32.4+20\log_{10}(\hat{f}_k^{carrier})+20\log_{10}(d_k)$ and we use $\hat{f}_k^{carrier}=2.6$ GHz. The available bandwidth of the system is 3 MHz with the noise power spectral density -140 dBm/Hz. The transmission power of the BS and workers are set as $p_{0}=0.5~W$ and $p_k=0.05 ~W$, respectively. The maximum edge CPU frequency is $f^{E,\max}=6$ GHz with the number of FLOPs per cycle $n^E=2$. The maximum CPU frequency of workers is randomly selected from $\{ 0.8, 1, 1.2 \}$ GHz with $n_k=1$. The effective capacitance coefficient is $\epsilon_k=2\times 10^{-28}$. The number of training data of workers is randomly selected from $\{ 2400, 3200, 4000 \}$, and the batch size is $b_k=16$. The training dataset is CIFAR-10 \cite{cifar}, and the size of each image is $32\times 32 \times 3$. The number of global rounds and local epochs are $\tau=50$ and $e_k=3$, respectively. The trained neural network is MobileNetV3-Large \cite{mobilenetv3}, which is a well-performing and lightweight convolutional neural network.

\subsubsection{Hyper-parameters of the Proposed Algorithm}
Three-layer fully connected neural network (FCNN) is utilized as the generator in the proposed algorithm. The number of the input nodes, two-layer intermediate nodes, and the output nodes are all 64. The discriminator is a two-layer FCNN with 64 nodes for the two inputs and 128 and 1 node for the middle and output layers respectively. The learning rates of generator and discriminator are both set as $4\times 10^{-4}$. The number of iterations is $M=10$. The losses are optimized with the Adam optimizer \cite{adam}. The number of generated solutions is $R=100$. The distance limitation of dominance pairs is $\gamma=80$. The number of dominance pairs corresponding to one dominating solution cannot exceed $\kappa=6$. 

\subsection{Performance Comparison}
\begin{figure*}[t]
\centering
\subfigure[]{
\includegraphics[width=7.5cm]{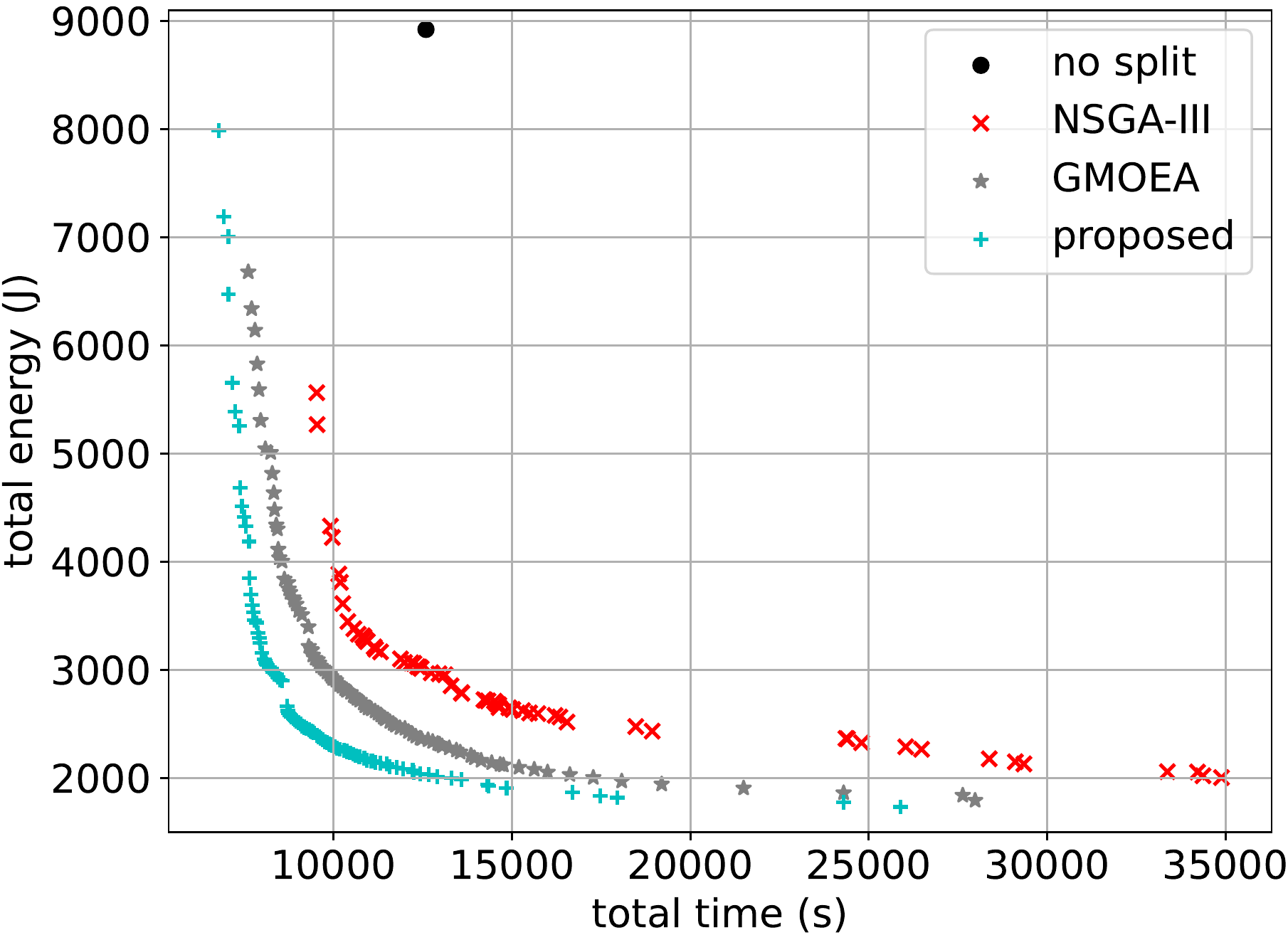}
}
\subfigure[]{
\includegraphics[width=7.5cm]{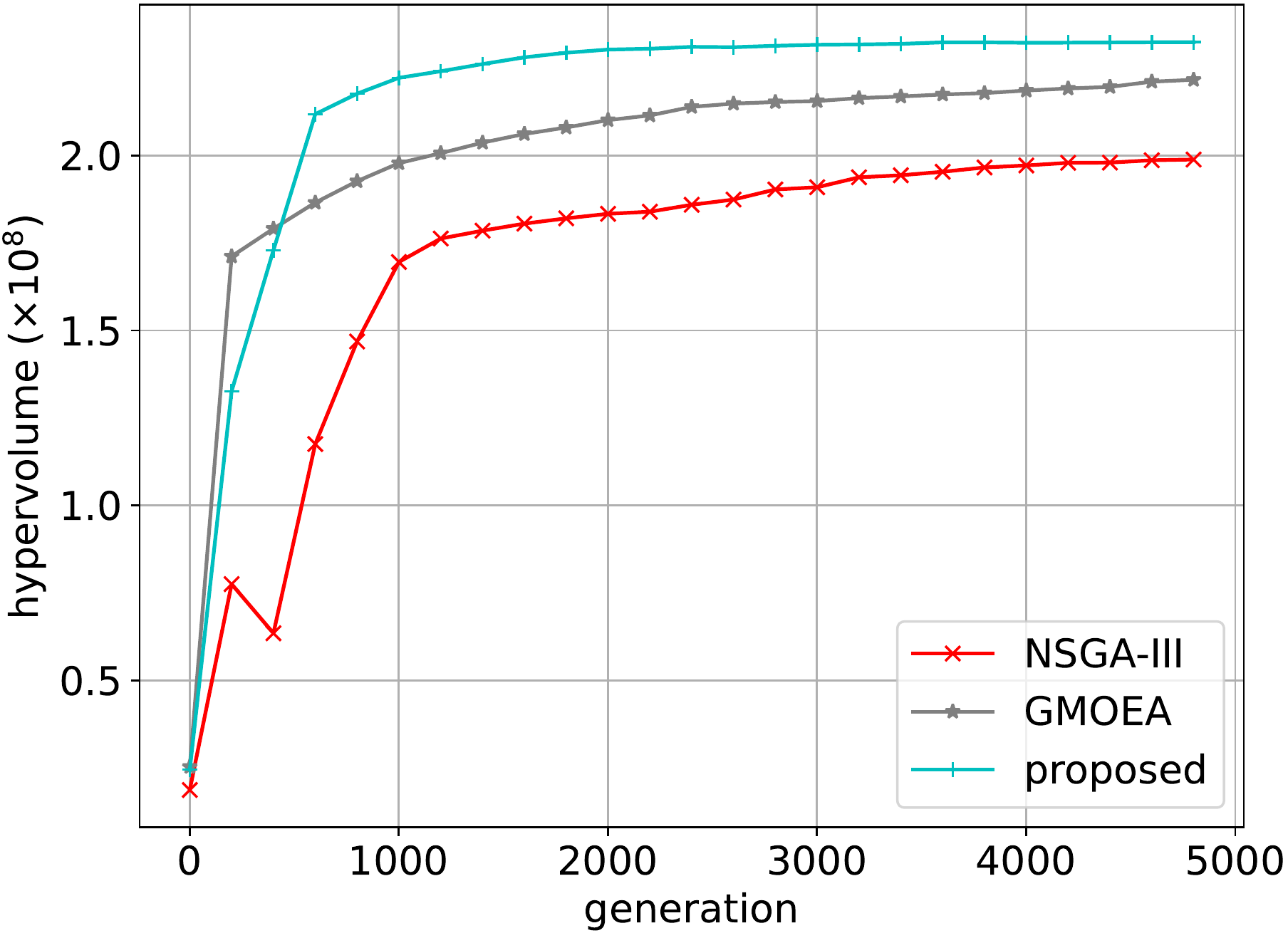}
}
\caption{(a) Comparison of Pareto fronts using different algorithms. (b) Convergence of different algorithms.}
\label{alg_comp}
\end{figure*}

\begin{figure*}[t]
\centering
\subfigure[]{
\includegraphics[width=7.5cm]{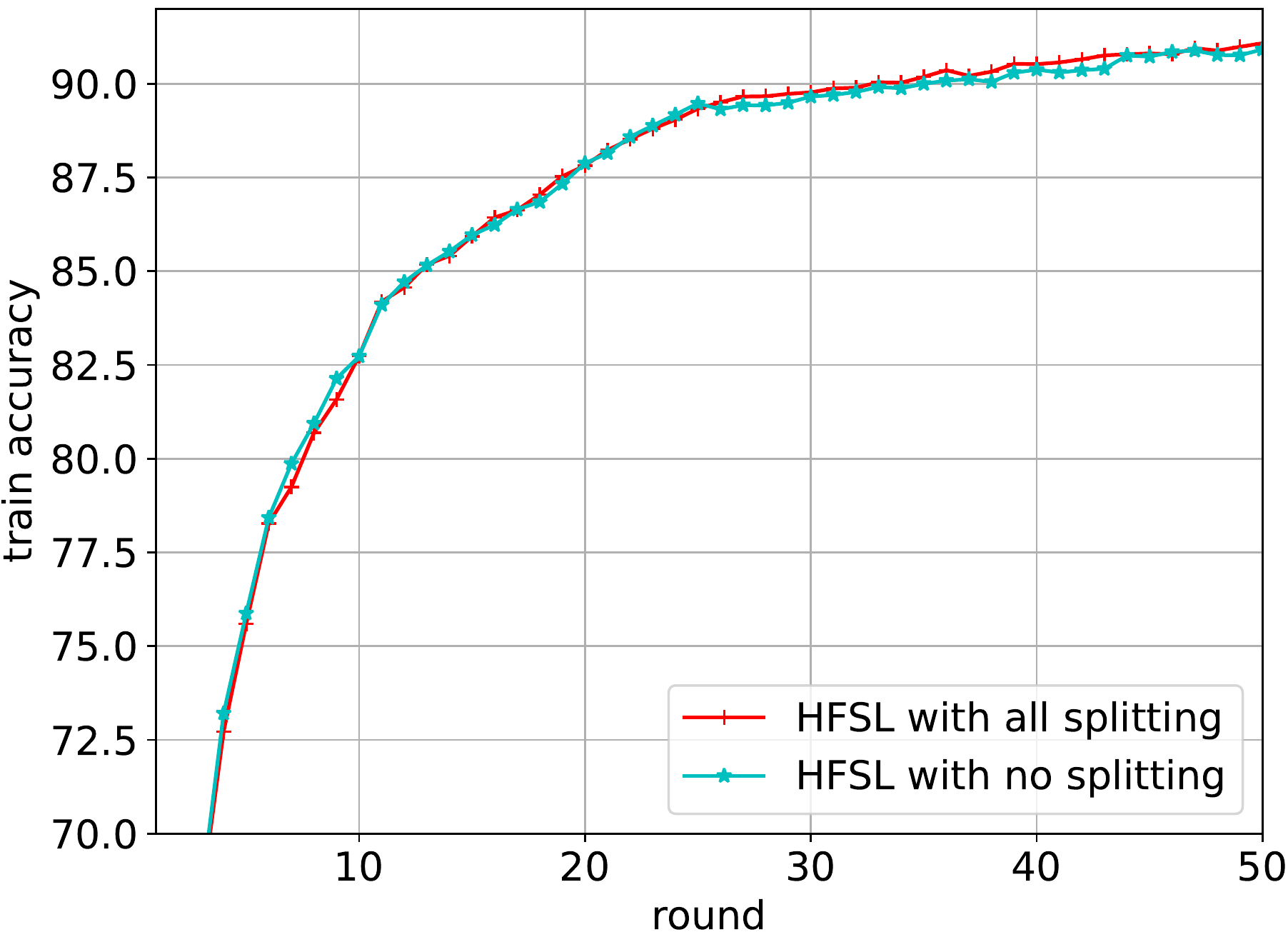}
}
\subfigure[]{
\includegraphics[width=7.5cm]{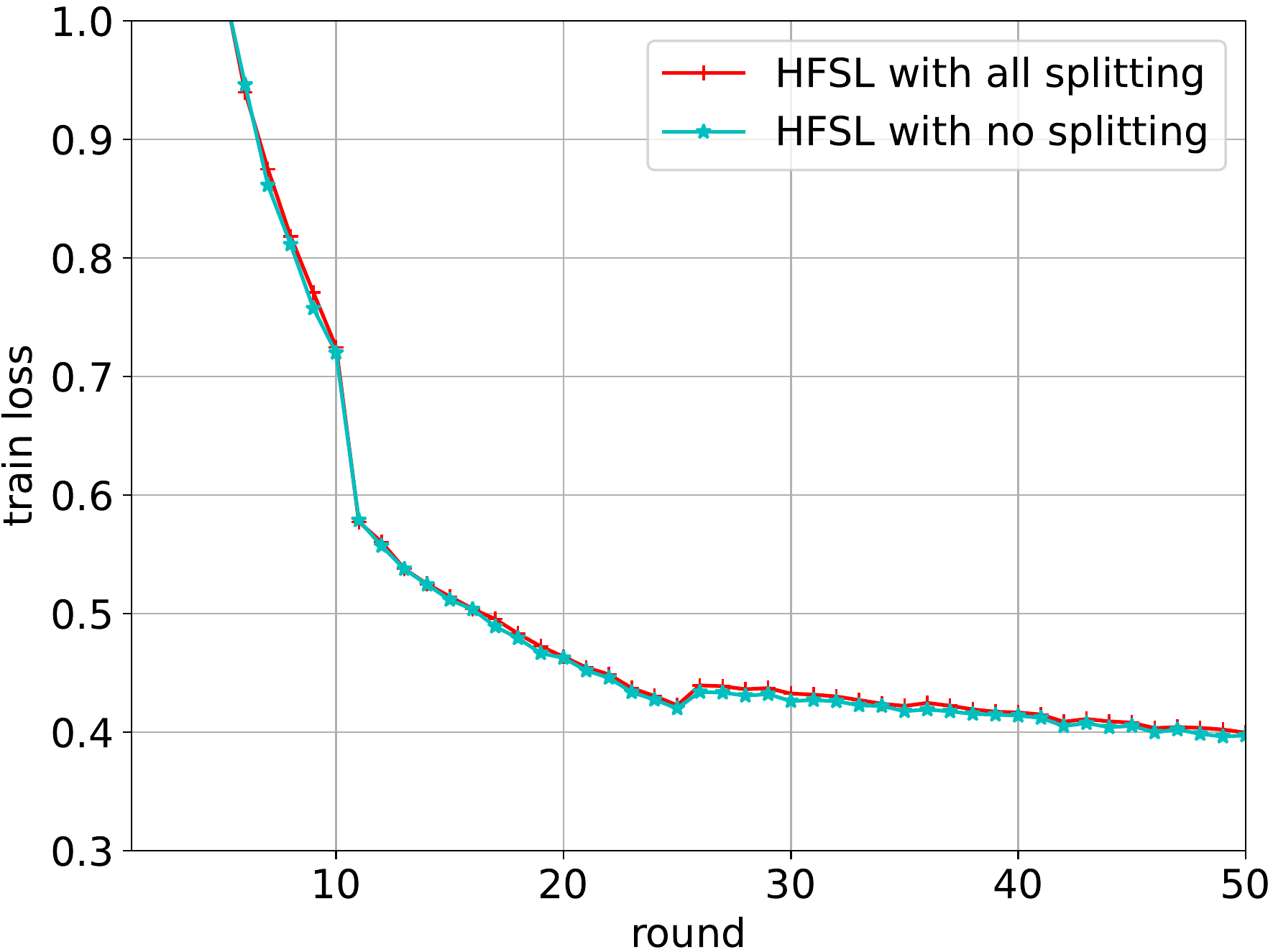}
}
\caption{(a) Train accuracy of HFSL with no model splitting and HFSL with all splitting. (b) Train loss of HFSL with no model splitting and HFSL with all splitting.}
\label{train_curve}
\end{figure*} 

To evaluate the performance of the proposed multi-objective optimization algorithm, we compare it with the NSGA-III \cite{NSGAIII} and GMOEA \cite{MOGAN}. For a fair comparison, the generator of GMOEA has the same number of nodes as the proposed algorithm. The discriminator of GMOEA is also a two-layer FCNN, and the number of the input nodes, intermediate nodes, and output nodes is 64, 64, and 1, respectively. The solution selection method of NSGA-III is applied to classify solutions in GMOEA. For the NSGA-III, the SBX \cite{deb1995simulated} and the polynomial mutation (PM) \cite{deb1996combined} are used to generate offspring solutions with the distribution index of crossover and mutation both set as 20. These three algorithms are performed for 5000 generations. 

Fig. \ref{alg_comp} (a) shows the Pareto fronts obtained by different algorithms. We can see that the performance of the Pareto front obtained by the proposed algorithm is better than the other two algorithms, which is reflected in the less energy consumption for the same training time, and less training time for the same energy consumption. Moreover, we compare the Pareto dominating solutions of the designed HFSL with the solution with no model splitting, i.e., FL. We can observe from the figure that the point without splitting is dominated by some solutions obtained by the three algorithms. 

In Fig. \ref{alg_comp} (b), we also use the hypervolume to evaluate the Pareto front \cite{while2006faster}. The hypervolume indicator refers to the area dominated by the point of Pareto front $\Omega$ and bounded above by a reference point $\bm{\omega}^{\ast}$
\begin{equation}
HV(\Omega,\bm{\omega}^{\ast})=\Lambda (\{ \bm{\omega}\in \mathbb{R}^2|\exists~ \bm{\omega}'\in \Omega : \bm{\omega}'\preceq \bm{\omega}~ and~ \bm{\omega} \preceq \bm{\omega}^{\ast} \} ),
\end{equation}
where $\Lambda (\cdot)$ is the Lebesgue measure. Note that a greater value of hypervolume indicates better performance. With the reference point $\bm{\omega}^{\ast}=[36000, 10000]$, we calculate the hypervolume of Pareto fronts during training, resulting in the convergence curves of different algorithms, as shown in Fig. \ref{alg_comp} (b). We can find that the proposed algorithm converges faster than the other algorithms and finally achieves a larger hypervolume.

The simulation runs on a Nvidia RTX 3080 GPU with Intel Core i7-11700 CPU. For the training time of these three algorithms, the time required for the proposed algorithm, GMOEA, and NSGA-III to train for 2000 generations (almost converged) is 515.7, 483.6, and 44 seconds, respectively. The time consumption of the proposed algorithm and GMOEA are approximate, and their training time is mainly spent on the training of the neural network. Although their time consumption is much larger than NSGA-III, since the training time of hybrid federated split learning system is several hours, a few minutes of optimization is acceptable.

As for the learning performance of HFSL, we compare the train accuracy and train loss of HFSL with no model splitting and HFSL with all splitting, as shown in Fig. \ref{train_curve}. The workers in HFSL with no splitting update the model with normal gradient, i.e., equation (\ref{grad}). In contrast, all workers in HFSL with all splitting update with the delayed gradient (\ref{delay_grad}). We can see from Fig. \ref{train_curve} that the train accuracy curves of HFSL with no splitting and HFSL with all splitting almost coincide, which verifies that the delayed gradient does not affect the convergence rate.

\subsection{Pareto Front versus the Bandwidth $B^{\max}$}

\begin{figure*}[t]
\centering
\subfigure[]{
\includegraphics[width=7.2cm]{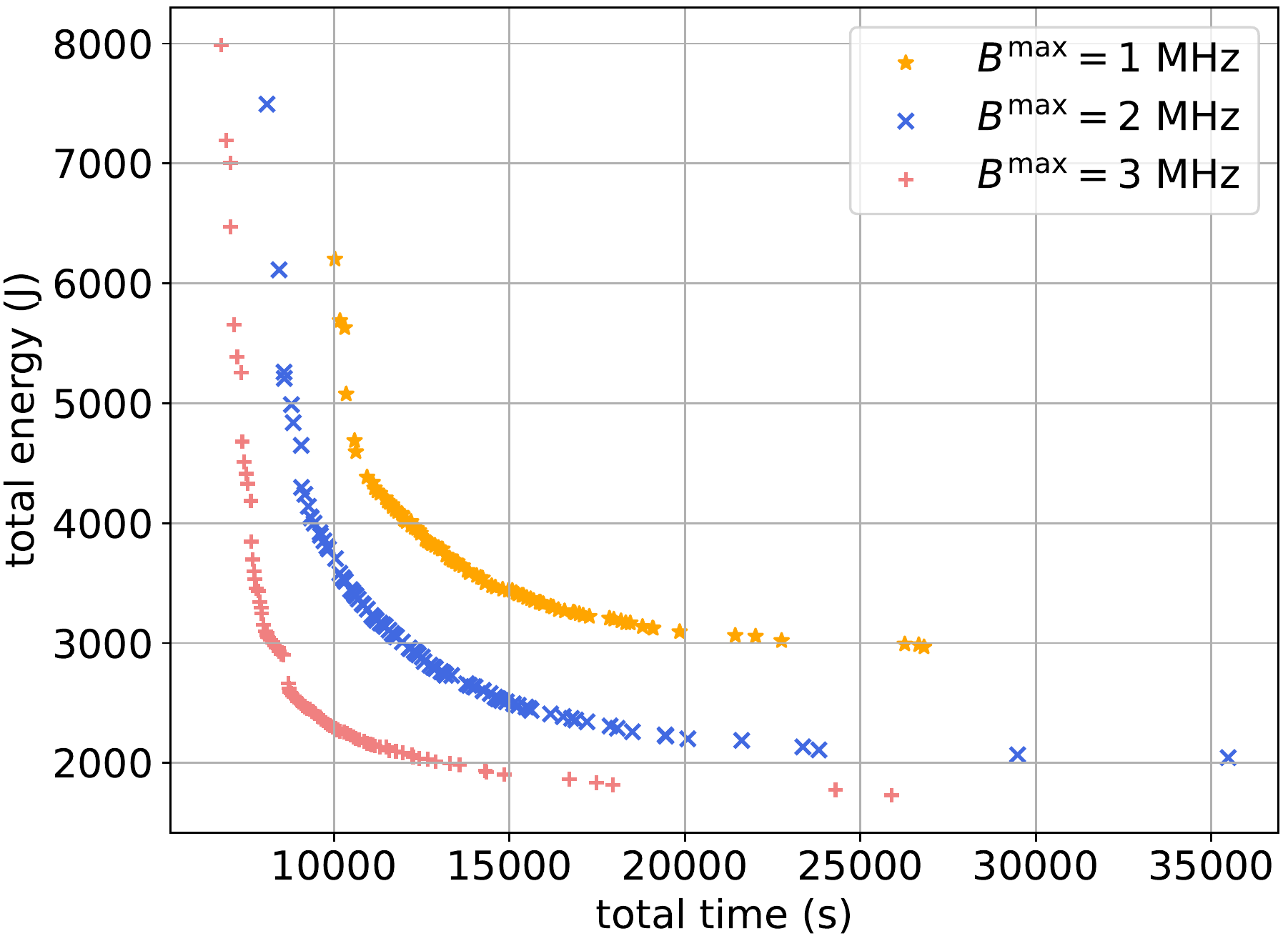}
}
\subfigure[]{
\includegraphics[width=7.4cm]{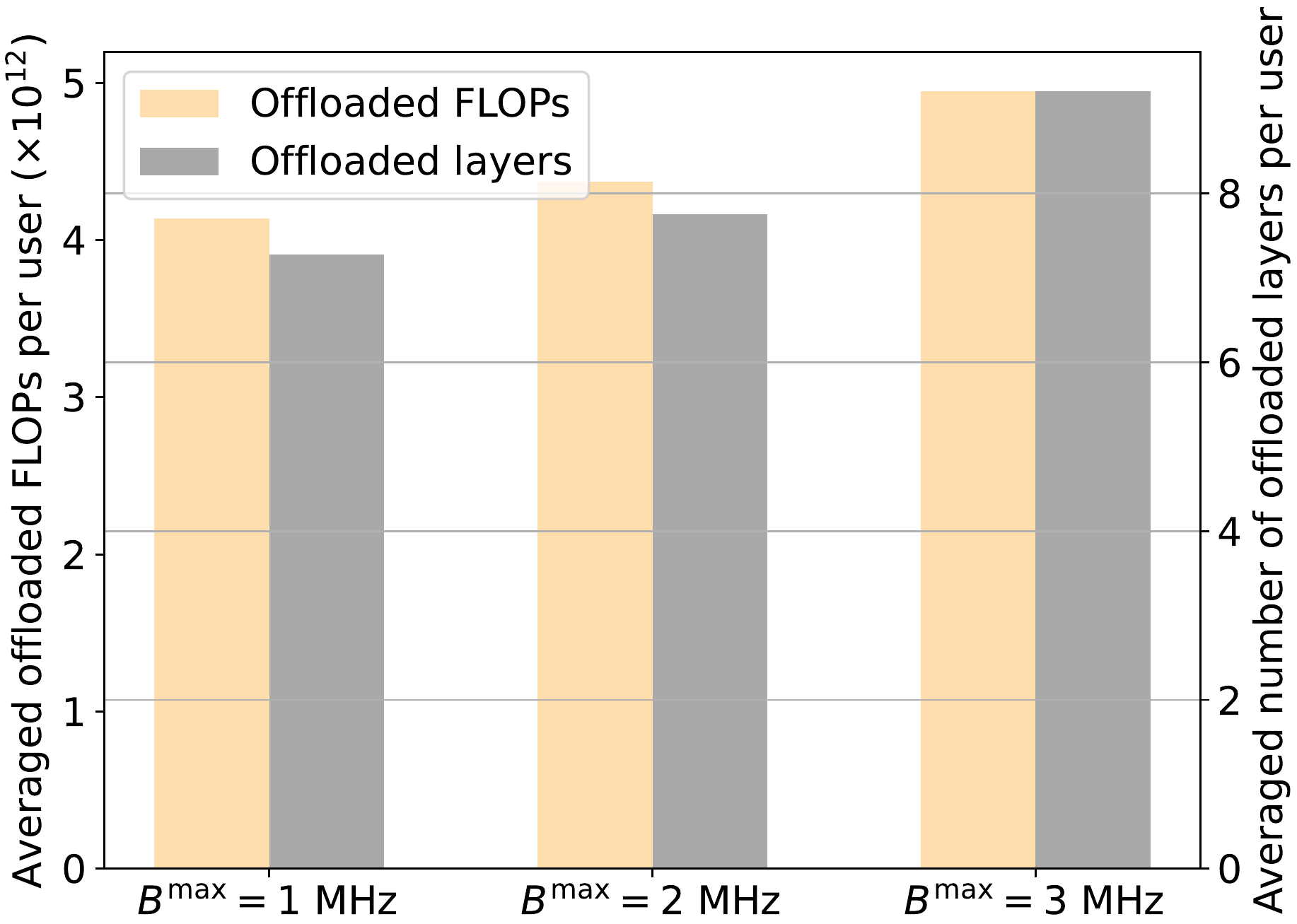}
}
\caption{(a) Pareto fronts of different bandwidth $B^{\max}$ obtained by the proposed algorithm. (b) Averaged number of FLOPs and layers offloaded to the server per user with different bandwidth $B^{\max}$.}
\label{diff_B}
\end{figure*}

\begin{figure*}[t]
\centering
\subfigure[]{
\includegraphics[width=7.2cm]{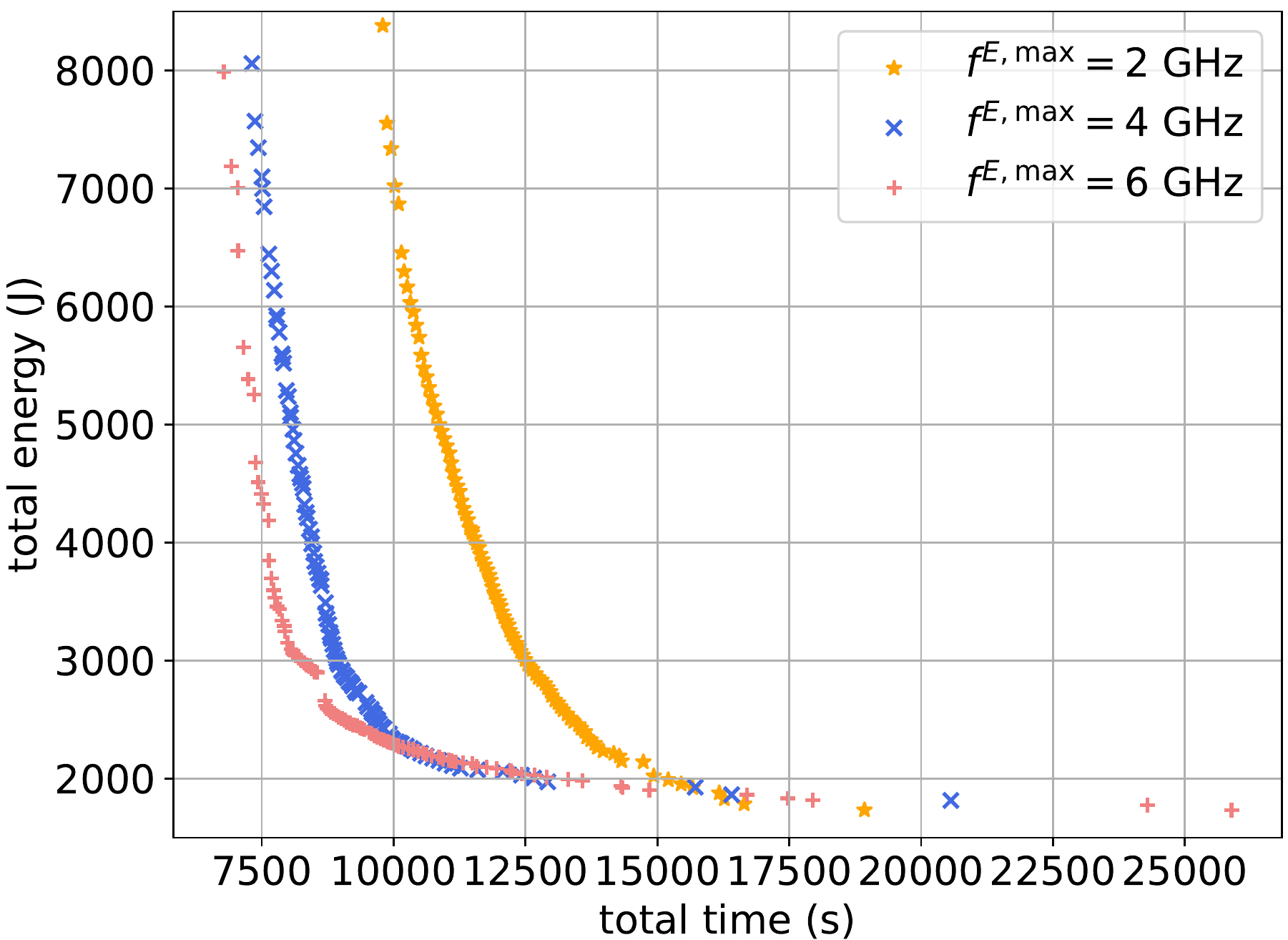}
}
\subfigure[]{
\includegraphics[width=7.4cm]{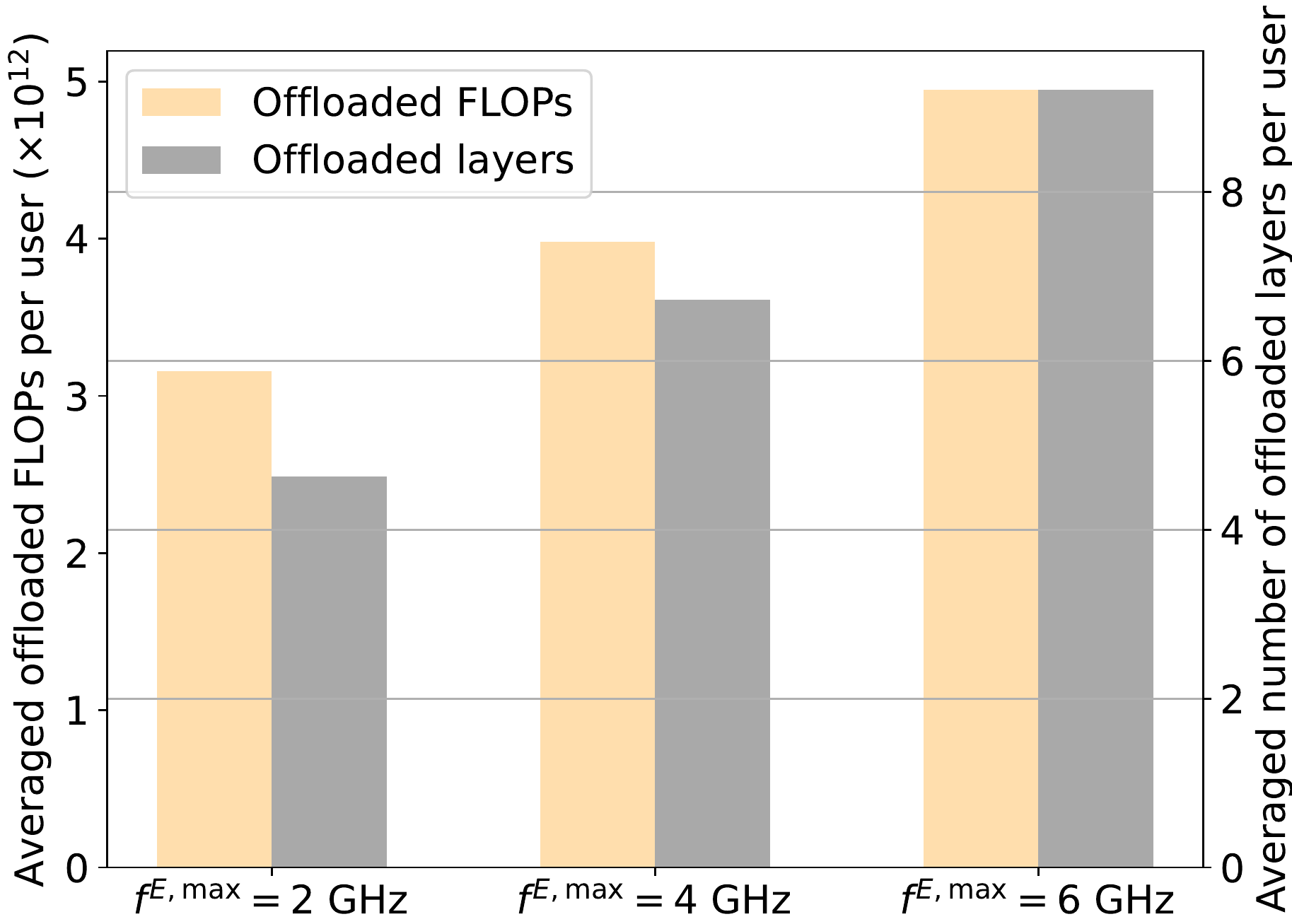}
}
\caption{(a) Pareto fronts of different CPU frequency $f^{E,\max}$ obtained by the proposed algorithm. (b) Averaged number of FLOPs and layers offloaded to the server per user with different CPU frequency $f^{E,\max}$.}
\label{diff_F}
\end{figure*} 

The impact of bandwidth $B^{\max}$ on Pareto fronts is shown in Fig. \ref{diff_B} (a) with $f^{E,\max}=6$ GHz. Increasing bandwidth can greatly reduce training time and energy consumption. On the one hand, the increase of $B^{\max}$ can reduce the transmission delay of parameters and intermediate results, while reducing the energy cost of the upload. On the other hand, with sufficient computing power of the edge server, increasing the bandwidth can increase the offloaded computing load, as shown in Fig. \ref{diff_B} (b), thereby reducing the energy consumption of workers.

\subsection{Pareto Front versus the CPU Frequency $f^{E,\max}$ of Edge Server }
The impact of computing resource $f^{E,\max}$ on Pareto fronts is shown in Fig. \ref{diff_F} (a) with $B^{\max}=3$ MHz. Increasing computing frequency of edge server can also reduce training time and energy consumption. As shown in Fig. \ref{diff_F} (b), when the bandwidth is sufficient, the number of FLOPs offloaded to the server increases with the increase of $f^{E,\max}$. Besides, the increase of $f^{E,\max}$ can reduce the computation delay of the server, thereby decreasing the training time.

\section{Conclusion}
\label{con}
In this paper, we proposed a hybrid federated split learning framework to utilize the multi-worker parallel update and low communication burden of FL and the low computational requirement for workers of SL to reduce the training time and energy consumption. To reduce the computational idleness of workers with model splitting, we designed a parallel computing scheme for model splitting without label sharing. Convergence analysis shows that the delayed gradient update introduced by the parallel computing scheme does not affect the convergence rate. Then we formulated a multi-objective optimization problem to find the Pareto-optimal solution set of splitting decisions and resource allocation for the proposed HFSL. To solve the problem, we proposed a predictive GAN-powered multi-objective optimization algorithm. Experimental results show that the proposed HFSL enables various trade-offs between training time and energy consumption, and the solutions of HFSL strictly dominate the solution of FL. Moreover, the Pareto front found by the proposed multi-objective algorithm outperforms other methods, e.g., GMOEA, and NSGA-III.

\begin{appendices}
\section{Proof of Lemma 1}
\label{lemma1}
\begin{align}
& \mathbb{E}\big[\| \overline{\bm{w}}_t^n-\bm{w}^n_{k,t}\|^2 \big ] =\mathbb{E}\big[\| \sum_{k=1}^K \frac{D_k}{D}\big ( I_k \bm{w}^n_{k,t} +(1-I_k)\hat{\bm{w}}^n_{k,t} \big ) -\bm{w}^n_{k,t}\|^2 \big ] \nonumber \\ 
&\stackrel{(a)}{=} \eta^2\mathbb{E}\big[\| \sum_{k=1}^K \frac{D_k}{D} \big ( I_k \sum_{n^{'}=1}^{n} \nabla F_k (\bm{w}^{n^{'}-1}_{k,t}) {\setlength\arraycolsep{0.3pt}+}(1-I_k)\sum_{n^{'}=1}^{n} \nabla F_k (\hat{\bm{w}}^{n^{'}-2}_{k,t}) \big ) {\setlength\arraycolsep{0.3pt}-} \sum_{n^{'}=1}^{n} \nabla F_k (\bm{w}^{n^{'}-1}_{k,t})\|^2 \big ] \nonumber \\
&\stackrel{(b)}{\leq} 2\eta^2\mathbb{E}\big[\| \sum_{n^{'}=1}^{n} \sum_{k=1}^K \frac{D_k}{D}\big ( I_k \nabla F_k (\bm{w}^{n^{'}-1}_{k,t})  +(1-I_k) \nabla F_k (\hat{\bm{w}}^{n^{'}-2}_{k,t}) \big ) \|^2 + \| \sum_{n^{'}=1}^{n} \nabla F_k (\bm{w}^{n^{'}-1}_{k,t})\|^2 \big ] \nonumber \\
&\stackrel{(c)}{\leq} 2\eta^2 n \sum_{n^{'}=1}^{n} \mathbb{E}\big[  \|  \sum_{k=1}^K \frac{D_k}{D}\big ( I_k \nabla F_k (\bm{w}^{n^{'}-1}_{k,t}) +(1-I_k) \nabla F_k (\hat{\bm{w}}^{n^{'}-2}_{k,t}) \big ) \|^2 + \| \nabla F_k (\bm{w}^{n^{'}-1}_{k,t})\|^2 \big ] \nonumber \\
&\stackrel{(d)}{\leq} 2\eta^2 n \sum_{n^{'}=1}^{n} \mathbb{E}\big[  \sum_{k=1}^K \frac{D_k}{D} \big ( I_k \| \nabla F_k (\bm{w}^{n^{'}-1}_{k,t}) \|^2  +(1-I_k) \| \nabla F_k (\hat{\bm{w}}^{n^{'}-2}_{k,t}) \|^2 \big ) + \| \nabla F_k (\bm{w}^{n^{'}-1}_{k,t})\|^2 \big ] \nonumber \\
&\stackrel{(e)}{\leq} 4\eta^2 n^2 G^2.
\end{align}
where (a) is obtained by  $\bm{w}^n_{k,t}=\bm{w}^{0}_{k,t}-\eta \sum_{n^{'}=1}^n \nabla F_k (\bm{w}^{n^{'}-1}_{k,t})$ and $\hat{\bm{w}}^n_{k,t}=\hat{\bm{w}}^{0}_{k,t}-\eta \sum_{n^{'}=1}^n \nabla F_k (\hat{\bm{w}}^{n^{'}-2}_{k,t})$. (b) and (c) are obtained by using the inequality $\|\sum_{n=1}^{N^{\max}} \bm{w}_n \|\leq {N^{\max}}\sum_{n=1}^{N^{\max}} \| \bm{w}_n \|$ for any vectors $\bm{w}_n$. (d) holds due to the convexity of $\| \cdot \|$. (e) follows from Assumption 3.

Similarly, we can obtain the following relationship
\begin{align}
& \mathbb{E}\big[\| \overline{\bm{w}}_t^n-\hat{\bm{w}}^{n-1}_{k,t} \|^2 \big ] \nonumber \\
&=\eta^2\mathbb{E}\big[\| \sum_{n^{'}=1}^{n} \sum_{k=1}^K \frac{D_k}{D}\big ( I_k \nabla F_k (\bm{w}^{n^{'}-1}_{k,t}) +(1-I_k) \nabla F_k (\hat{\bm{w}}^{n^{'}-2}_{k,t}) \big ) - \sum_{n^{'}=1}^{n-1} \nabla F_k (\hat{\bm{w}}^{n^{'}-2}_{k,t})\|^2 \big ] \nonumber \\
&\leq 2\eta^2\mathbb{E}\big[\| \sum_{n^{'}=1}^{n} \sum_{k=1}^K \frac{D_k}{D}\big ( I_k \nabla F_k (\bm{w}^{n^{'}-1}_{k,t}) +(1-I_k) \nabla F_k (\hat{\bm{w}}^{n^{'}-2}_{k,t}) \big ) \|^2 + \| \sum_{n^{'}=1}^{n-1} \nabla F_k (\hat{\bm{w}}^{n^{'}-2}_{k,t}) \big ] \nonumber \\
&\leq 2\eta^2  \mathbb{E}\big[ n \sum_{n^{'}=1}^{n} \|  \sum_{k=1}^K \frac{D_k}{D}\big ( I_k \nabla F_k (\bm{w}^{n^{'}-1}_{k,t}) {\setlength\arraycolsep{0.3pt}+}(1-I_k) \nabla F_k (\hat{\bm{w}}^{n^{'}-2}_{k,t}) \big ) \|^2 {\setlength\arraycolsep{0.3pt}+} (n-1) \sum_{n^{'}=1}^{n-1} \| \nabla F_k (\hat{\bm{w}}^{n^{'}-2}_{k,t})\|^2 \big ] \nonumber \\
&\leq 2\eta^2  \mathbb{E}\big[ n \sum_{n^{'}=1}^{n}   \sum_{k=1}^K \frac{D_k}{D}\big ( I_k \| \nabla F_k (\bm{w}^{n^{'}-1}_{k,t})\|^2 {\setlength\arraycolsep{0.3pt}+}(1{\setlength\arraycolsep{0.3pt}-}I_k) \| \nabla F_k (\hat{\bm{w}}^{n^{'}-2}_{k,t}) \|^2 \big ) {\setlength\arraycolsep{0.3pt}+} (n{\setlength\arraycolsep{0.3pt}-}1) \sum_{n^{'}=1}^{n-1} \| \nabla F_k (\hat{\bm{w}}^{n^{'}-2}_{k,t})\|^2 \big ] \nonumber \\
&\leq 2\eta^2 (n^2 +(n-1)^2 )G^2.
\end{align}

\section{Proof of Theorem 1}
\label{theorem1}
Based on the $L$-smoothness of $F_k(\cdot)$, it can be proved that $F(\cdot)$ is also $L$-smooth. So we have
\begin{align}
\mathbb{E}[F(\overline{\bm{w}}_t^n)] \leq \mathbb{E}\big [F(\overline{\bm{w}}_t^{n-1})+& \langle \nabla F(\overline{\bm{w}}_t^{n-1}), \overline{\bm{w}}_t^n-\overline{\bm{w}}_t^{n-1} \rangle +\frac{L}{2} \| \overline{\bm{w}}_t^n-\overline{\bm{w}}_t^{n-1}\|^{2} \big ].
\end{align}
For the last term on the right-hand side of the above inequality, we have
\begin{align}
 \mathbb{E}\big [\| \overline{\bm{w}}_t^n-\overline{\bm{w}}_t^{n-1}\|^{2} \big ] &= \mathbb{E}\big [\| \sum_{k=1}^K \frac{D_k}{D}\big (I_k (\bm{w}^n_{k,t}{\setlength\arraycolsep{0.3pt}-}\bm{w}^{n-1}_{k,t}) {\setlength\arraycolsep{0.3pt}+}(1{\setlength\arraycolsep{0.3pt}-} I_k)(\hat{\bm{w}}^n_{k,t}{\setlength\arraycolsep{0.3pt}-}\hat{\bm{w}}^{n-1}_{k,t}) \big ) \|^{2} \big ] \nonumber \\
&=\eta^2 \mathbb{E}\big [\| \sum_{k=1}^K \frac{D_k}{D}\big (I_k \nabla F_k (\bm{w}^{n-1}_{k,t}) {\setlength\arraycolsep{0.3pt}+}(1 {\setlength\arraycolsep{0.3pt}-} I_k)\nabla F_k (\hat{\bm{w}}^{n-2}_{k,t}) \big ) \|^{2} \big ].
\end{align}

Moreover,
\begin{align}
&\mathbb{E}\big[\langle \nabla F(\overline{\bm{w}}_t^{n-1}), \overline{\bm{w}}_t^n-\overline{\bm{w}}_t^{n-1} \rangle \big ] \nonumber \\
&= -\eta \mathbb{E}\big[\langle \nabla F(\overline{\bm{w}}_t^{n-1}), \sum_{k=1}^K \frac{D_k}{D}\big ( I_k \nabla F_k (\bm{w}^{n-1}_{k,t}) +(1-I_k)\nabla F_k (\hat{\bm{w}}^{n-2}_{k,t}) \big ) \rangle \big ] \nonumber \\
&\stackrel{(a)}{=} \frac{\eta}{2} \mathbb{E}\big[\| \nabla F(\overline{\bm{w}}_t^{n-1}) - \sum_{k=1}^K \frac{D_k}{D}\big ( I_k \nabla F_k (\bm{w}^{n-1}_{k,t}) +(1-I_k)\nabla F_k (\hat{\bm{w}}^{n-2}_{k,t}) \big ) \|^2 \big ]  - \frac{\eta}{2} \mathbb{E}\big[\| \nabla F(\overline{\bm{w}}_t^{n-1}) \|^2 \big ] \nonumber \\
& \quad{\setlength\arraycolsep{0.3pt}-}\frac{\eta}{2} \mathbb{E}\big[\| \sum_{k=1}^K \frac{D_k}{D}\big ( I_k \nabla F_k (\bm{w}^{n-1}_{k,t}) {\setlength\arraycolsep{0.3pt}+}(1{\setlength\arraycolsep{0.3pt}-}I_k)\nabla F_k (\hat{\bm{w}}^{n-2}_{k,t}) \big ) \|^2 \big ],
\label{innermul}
\end{align}
where (a) follows from the fact that $\langle \bm{w}_1,\bm{w}_2 \rangle=\frac{1}{2}(\| \bm{w}_1 \|^2+\| \bm{w}_2 \|^2-\| \bm{w}_1-\bm{w}_2 \|^2)$ for any two vectors $\bm{w}_1$ and $\bm{w}_2$ of the same size.

For the first term in equation (\ref{innermul}), we have
\begin{align}
& \mathbb{E}\big[\| \nabla F(\overline{\bm{w}}_t^{n-1}) - \sum_{k=1}^K \frac{D_k}{D}\big ( I_k \nabla F_k (\bm{w}^{n-1}_{k,t}) +(1-I_k)\nabla F_k (\hat{\bm{w}}^{n-2}_{k,t}) \big ) \|^2 \big ] \nonumber \\
&= \mathbb{E}\big[\| \sum_{k=1}^K \frac{D_k}{D}\big (I_k (\nabla F_k(\overline{\bm{w}}_t^{n-1})-\nabla F_k(\bm{w}^{n-1}_{k,t})) +(1-I_k)(\nabla F_k(\overline{\bm{w}}_t^{n-1})-\nabla F_k (\hat{\bm{w}}^{n-2}_{k,t})) \big ) \|^2 \big ] \nonumber \\
&\stackrel{(a)}{\leq} \sum_{k=1}^K \frac{D_k}{D} \Big ( I_k \mathbb{E}\big[\| \nabla F_k(\overline{\bm{w}}_t^{n-1})-\nabla F_k(\bm{w}^{n-1}_{k,t}) \|^2 \big ] +(1-I_k)\mathbb{E}\big[\| \nabla F_k(\overline{\bm{w}}_t^{n-1})-\nabla F_k (\hat{\bm{w}}^{n-2}_{k,t})  \|^2 \big ] \Big ) \nonumber \\
&\stackrel{(b)}{\leq} L^2 \sum_{k=1}^K \frac{D_k}{D} \Big ( I_k \mathbb{E}\big[\|  \overline{\bm{w}}_t^{n-1}-\bm{w}^{n-1}_{k,t}\|^2 \big ] +(1-I_k)\mathbb{E}\big[\| \overline{\bm{w}}_t^{n-1}-\hat{\bm{w}}^{n-2}_{k,t} \|^2 \big ] \Big )\nonumber \\
&\stackrel{(c)}{\leq} 2 \eta^2 G^2 L^2 \sum_{k=1}^K \frac{D_k}{D}\big ( 2 I_k (n{\setlength\arraycolsep{0.3pt}-}1)^2 {\setlength\arraycolsep{0.3pt}+}(1{\setlength\arraycolsep{0.3pt}-}I_k)((n{\setlength\arraycolsep{0.3pt}-}1)^2{\setlength\arraycolsep{0.3pt}+}(n{\setlength\arraycolsep{0.3pt}-}2)^2) \big ),
\end{align}
where (a) is obtained by Jensen's inequality. (b) holds because $L$-smoothness of $F_k(\cdot)$ can also be expressed as $\| \nabla F_k(\bm{w}_2)-\nabla F_k(\bm{w}_1) \| \leq L \| \bm{w}_2-\bm{w}_1 \|, ~\forall \bm{w}_1,\bm{w}_2$. (c) follows from Lemma 1. 

Combining the above inequalities together, we have
\begin{align}
\mathbb{E}[F(\overline{\bm{w}}_t^n)] \leq \mathbb{E}&\big [F(\overline{\bm{w}}_t^{n-1})- \frac{\eta}{2} \| \nabla F(\overline{\bm{w}}_t^{n-1}) \|^2+ \alpha(n) \nonumber \\ 
&{\setlength\arraycolsep{0.3pt}+}\frac{L\eta^2{\setlength\arraycolsep{0.3pt}-}\eta}{2} \sum_{k=1}^K \frac{D_k}{D}\big (I_k \nabla F_k (\bm{w}^{n-1}_{k,t}) {\setlength\arraycolsep{0.3pt}+}(1{\setlength\arraycolsep{0.3pt}-}I_k)\nabla F_k (\hat{\bm{w}}^{n-2}_{k,t}) \big ) \|^{2} \big ] ,
\label{bound1}
\end{align}
where $\alpha(n)= \eta^3 G^2 L^2 \sum_{k=1}^K \frac{D_k}{D}\big ( 2 I_k (n-1)^2 +(1-I_k)((n-1)^2+(n-2)^2) \big )$.

By minimizing both sides of the inequality (\ref{convex}) with respect to $\bm{w}_2$, we can obtain that $ F(\bm{w}^{\ast})\geq F(\bm{w}_1)-\frac{1}{2\mu}\| \nabla F(\bm{w}_1) \|^2$ for any  $\bm{w}_1$. So we have 
\begin{equation}
\| \nabla F(\overline{\bm{w}}_t^{n-1}) \|^2 \geq 2\mu \big ( F(\overline{\bm{w}}_t^{n-1})- F(\bm{w}^{\ast})\big ).
\label{mu_convex}
\end{equation}
By setting the learning rate $\eta \leq \frac{1}{L}$ and substituting the inequality (\ref{mu_convex}) into (\ref{bound1}), we have
\begin{align}
\mathbb{E}[F(\overline{\bm{w}}_t^n)] {\setlength\arraycolsep{0.3pt}-} F(\bm{w}^{\ast}) &\leq (1{\setlength\arraycolsep{0.3pt}-}\mu \eta)\mathbb{E}\big [F(\overline{\bm{w}}_t^{n-1}){\setlength\arraycolsep{0.3pt}-} F(\bm{w}^{\ast}) \big ] {\setlength\arraycolsep{0.3pt}+} \alpha (n).
\label{bound2}
\end{align}
By using the inequality (\ref{bound2}) recursively, we have 
\begin{align}
&\mathbb{E}[F(\bm{W}_t)] - F(\bm{w}^{\ast})=\mathbb{E}[F(\overline{\bm{w}}_t^n)] - F(\bm{w}^{\ast}) \nonumber \\
&\leq (1-\mu \eta)^{N^{\max}} \mathbb{E}\big [F(\overline{\bm{w}}_t^0)- F(\bm{w}^{\ast}) \big ] + \sum_{n=0}^{{N^{\max}}-1}(1-\nu \eta)^{n} \alpha ({N^{\max}}-n) \nonumber \\
&= (1-\nu \eta)^{N^{\max}} \mathbb{E}\big [F(\bm{W}_{t-1})- F(\bm{w}^{\ast}) \big ] + \sum_{n=0}^{{N^{\max}}-1}(1-\nu \eta)^{n} \alpha ({N^{\max}}-n).
\label{bound3}
\end{align}
By using the inequality (\ref{bound3}) recursively, we have
\begin{align}
\mathbb{E}[F(\bm{W}_{\tau})] - F(\bm{w}^{\ast}) \leq  (1-\nu \eta)^{{N^{\max}}\tau}& \mathbb{E}\big [F(\bm{W}_{0})- F(\bm{w}^{\ast}) \big ] \nonumber\\
&+ \sum_{t=0}^{\tau-1} (1-\nu \eta)^{N^{\max}t} \sum_{n=0}^{{N^{\max}}-1}(1-\nu \eta)^{n} \alpha ({N^{\max}}-n).
\end{align}

\end{appendices}

\bibliography{ref}

\begin{thebibliography}{10}
\providecommand{\url}[1]{#1}
\csname url@samestyle\endcsname
\providecommand{\newblock}{\relax}
\providecommand{\bibinfo}[2]{#2}
\providecommand{\BIBentrySTDinterwordspacing}{\spaceskip=0pt\relax}
\providecommand{\BIBentryALTinterwordstretchfactor}{4}
\providecommand{\BIBentryALTinterwordspacing}{\spaceskip=\fontdimen2\font plus
\BIBentryALTinterwordstretchfactor\fontdimen3\font minus
  \fontdimen4\font\relax}
\providecommand{\BIBforeignlanguage}[2]{{%
\expandafter\ifx\csname l@#1\endcsname\relax
\typeout{** WARNING: IEEEtran.bst: No hyphenation pattern has been}%
\typeout{** loaded for the language `#1'. Using the pattern for}%
\typeout{** the default language instead.}%
\else
\language=\csname l@#1\endcsname
\fi
#2}}
\providecommand{\BIBdecl}{\relax}
\BIBdecl

\bibitem{zhu2020toward}
G.~Zhu, D.~Liu, Y.~Du, C.~You, J.~Zhang, and K.~Huang, ``Toward an intelligent
  edge: Wireless communication meets machine learning,'' \emph{IEEE
  Communications Magazine}, vol.~58, no.~1, pp. 19--25, Jan. 2020.

\bibitem{wang2020convergence}
X.~Wang, Y.~Han, V.~C. Leung, D.~Niyato, X.~Yan, and X.~Chen, ``Convergence of
  edge computing and deep learning: A comprehensive survey,'' \emph{IEEE
  Communications Surveys \& Tutorials}, vol.~22, no.~2, pp. 869--904,
  Secondquarter 2020.

\bibitem{zhang2019deep}
C.~Zhang, P.~Patras, and H.~Haddadi, ``Deep learning in mobile and wireless
  networking: A survey,'' \emph{IEEE Communications Surveys \& Tutorials},
  vol.~21, no.~3, pp. 2224--2287, Thirdquarter 2019.

\bibitem{Chen2019}
M.~Chen, U.~Challita, W.~Saad, C.~Yin, and M.~Debbah, ``Artificial neural
  networks-based machine learning for wireless networks: A tutorial,''
  \emph{IEEE Communications Surveys \& Tutorials}, vol.~21, no.~4, pp.
  3039--3071, Fourthquarter 2019.

\bibitem{zhou2019edge}
Z.~Zhou, X.~Chen, E.~Li, L.~Zeng, K.~Luo, and J.~Zhang, ``Edge intelligence:
  Paving the last mile of artificial intelligence with edge computing,''
  \emph{Proceedings of the IEEE}, vol. 107, no.~8, pp. 1738--1762, Aug. 2019.

\bibitem{zap2019}
A.~Zappone, M.~Di~Renzo, and M.~Debbah, ``Wireless networks design in the era
  of deep learning: Model-based, {AI}-based, or both?'' \emph{IEEE Transactions
  on Communications}, vol.~67, no.~10, pp. 7331--7376, Oct. 2019.

\bibitem{Deng2020}
S.~Deng, H.~Zhao, W.~Fang, J.~Yin, S.~Dustdar, and A.~Y. Zomaya, ``Edge
  intelligence: The confluence of edge computing and artificial intelligence,''
  \emph{IEEE Internet of Things Journal}, vol.~7, no.~8, pp. 7457--7469, Aug.
  2020.

\bibitem{li2018learning}
H.~Li, K.~Ota, and M.~Dong, ``Learning {IoT} in edge: Deep learning for the
  {Internet of Things} with edge computing,'' \emph{IEEE Network}, vol.~32,
  no.~1, pp. 96--101, Feb. 2018.

\bibitem{data_nas}
B.~Yin, Z.~Chen, and M.~Tao, ``Dynamic data collection and neural architecture
  search for wireless edge intelligence systems,'' \emph{to appear in IEEE
  Transactions on Wireless Communications}, Aug. 2022.

\bibitem{federated}
B.~McMahan, E.~Moore, D.~Ramage, S.~Hampson, and B.~A. y~Arcas,
  ``Communication-efficient learning of deep networks from decentralized
  data,'' in \emph{Proc. Artificial intelligence and statistics}, 2017, pp.
  1273--1282.

\bibitem{gupta2018distributed}
O.~Gupta and R.~Raskar, ``Distributed learning of deep neural network over
  multiple agents,'' \emph{Journal of Network and Computer Applications}, vol.
  116, pp. 1--8, 2018.

\bibitem{vepakomma2018split}
P.~Vepakomma, O.~Gupta, T.~Swedish, and R.~Raskar, ``Split learning for health:
  Distributed deep learning without sharing raw patient data,'' \emph{arXiv
  preprint arXiv:1812.00564}, 2018.

\bibitem{HiveMind}
S.~Wang, X.~Zhang, H.~Uchiyama, and H.~Matsuda, ``Hivemind: Towards cellular
  native machine learning model splitting,'' \emph{IEEE Journal on Selected
  Areas in Communications}, vol.~40, no.~2, pp. 626--640, 2022.

\bibitem{thapa2020splitfed}
C.~Thapa, M.~A.~P. Chamikara, S.~Camtepe, and L.~Sun, ``Splitfed: When
  federated learning meets split learning,'' in \emph{Proc. Association for the
  Advancement of Artificial Intelligence}, 2022.

\bibitem{9652119}
Y.~Gao, M.~Kim, C.~Thapa, S.~Abuadbba, Z.~Zhang, S.~Camtepe, H.~Kim, and
  S.~Nepal, ``Evaluation and optimization of distributed machine learning
  techniques for internet of things,'' \emph{IEEE Transactions on Computers},
  2022.

\bibitem{han2021accelerating}
D.-J. Han, H.~I. Bhatti, J.~Lee, and J.~Moon, ``Accelerating federated learning
  with split learning on locally generated losses,'' in \emph{Proc. ICML
  Workshop on Federated Learning for User Privacy and Data Confidentiality},
  2021.

\bibitem{he2020group}
C.~He, M.~Annavaram, and S.~Avestimehr, ``Group knowledge transfer: Federated
  learning of large cnns at the edge,'' \emph{Advances in Neural Information
  Processing Systems}, vol.~33, pp. 14\,068--14\,080, 2020.

\bibitem{tian2022fedbert}
Y.~Tian, Y.~Wan, L.~Lyu, D.~Yao, H.~Jin, and L.~Sun, ``Fedbert: When federated
  learning meets pre-training,'' \emph{ACM Transactions on Intelligent Systems
  and Technology}, 2022.

\bibitem{park2021federated}
S.~Park, G.~Kim, J.~Kim, B.~Kim, and J.~C. Ye, ``Federated split vision
  transformer for covid-19cxr diagnosis using task-agnostic training,''
  \emph{arXiv preprint arXiv:2111.01338}, 2021.

\bibitem{pipedream}
D.~Narayanan, A.~Harlap, A.~Phanishayee, V.~Seshadri, N.~R. Devanur, G.~R.
  Ganger, P.~B. Gibbons, and M.~Zaharia, ``Pipedream: generalized pipeline
  parallelism for dnn training,'' in \emph{Proc. ACM Symposium on Operating
  Systems Principles}, 2019, pp. 1--15.

\bibitem{wang2020geryon}
S.~Wang, D.~Li, and J.~Geng, ``Geryon: Accelerating distributed cnn training by
  network-level flow scheduling,'' in \emph{Proc. IEEE Conference on Computer
  Communications}, 2020, pp. 1678--1687.

\bibitem{DynaComm}
S.~Cai, D.~Wang, H.~Wang, Y.~Lyu, G.~Xu, X.~Zheng, and A.~V. Vasilakos,
  ``Dynacomm: Accelerating distributed cnn training between edges and clouds
  through dynamic communication scheduling,'' \emph{IEEE Journal on Selected
  Areas in Communications}, vol.~40, no.~2, pp. 611--625, 2022.

\bibitem{wang2021overlap}
S.~Wang, A.~Pi, X.~Zhou, J.~Wang, and C.-Z. Xu, ``Overlapping communication
  with computation in parameter server for scalable dl training,'' \emph{IEEE
  Transactions on Parallel and Distributed Systems}, vol.~32, no.~9, pp.
  2144--2159, Sept. 2021.

\bibitem{zhang2007moea}
Q.~Zhang and H.~Li, ``Moea/d: A multiobjective evolutionary algorithm based on
  decomposition,'' \emph{IEEE Transactions on evolutionary computation},
  vol.~11, no.~6, pp. 712--731, 2007.

\bibitem{NSGAIII}
K.~Deb and H.~Jain, ``An evolutionary many-objective optimization algorithm
  using reference-point-based nondominated sorting approach, part i: Solving
  problems with box constraints,'' \emph{IEEE Transactions on Evolutionary
  Computation}, vol.~18, no.~4, pp. 577--601, Aug. 2014.

\bibitem{zhang2016decision}
X.~Zhang, Y.~Tian, R.~Cheng, and Y.~Jin, ``A decision variable clustering-based
  evolutionary algorithm for large-scale many-objective optimization,''
  \emph{IEEE Transactions on Evolutionary Computation}, vol.~22, no.~1, pp.
  97--112, 2016.

\bibitem{cheng2015multiobjective}
R.~Cheng, Y.~Jin, K.~Narukawa, and B.~Sendhoff, ``A multiobjective evolutionary
  algorithm using gaussian process-based inverse modeling,'' \emph{IEEE
  Transactions on Evolutionary Computation}, vol.~19, no.~6, pp. 838--856,
  2015.

\bibitem{8281523}
L.~Pan, C.~He, Y.~Tian, H.~Wang, X.~Zhang, and Y.~Jin, ``A classification-based
  surrogate-assisted evolutionary algorithm for expensive many-objective
  optimization,'' \emph{IEEE Transactions on Evolutionary Computation},
  vol.~23, no.~1, pp. 74--88, 2019.

\bibitem{6252865}
C.-W. Seah, Y.-S. Ong, I.~W. Tsang, and S.~Jiang, ``Pareto rank learning in
  multi-objective evolutionary algorithms,'' in \emph{Proc. IEEE Congress on
  Evolutionary Computation}, 2012, pp. 1--8.

\bibitem{MOGAN}
C.~He, S.~Huang, R.~Cheng, K.~C. Tan, and Y.~Jin, ``Evolutionary multiobjective
  optimization driven by generative adversarial networks (gans),'' \emph{IEEE
  Transactions on Cybernetics}, vol.~51, no.~6, pp. 3129--3142, June 2021.

\bibitem{wang2021manifold}
Z.~Wang, H.~Hong, K.~Ye, G.-E. Zhang, M.~Jiang, and K.~C. Tan, ``Manifold
  interpolation for large-scale multiobjective optimization via generative
  adversarial networks,'' \emph{IEEE Transactions on Neural Networks and
  Learning Systems}, 2022.

\bibitem{moppo}
Z.~Chen, B.~Yin, H.~Zhu, Y.~Li, M.~Tao, and W.~Zhang, ``Mobile communications,
  computing and caching resources allocation for diverse services via
  multi-objective proximal policy optimization,'' \emph{IEEE Transactions on
  Communications}, vol.~70, no.~7, pp. 4498--4512, July 2022.

\bibitem{burd1996processor}
T.~D. Burd and R.~W. Brodersen, ``Processor design for portable systems,''
  \emph{Journal of VLSI signal processing systems for signal, image and video
  technology}, vol.~13, no. 2-3, pp. 203--221, 1996.

\bibitem{zhu2021delayed}
L.~Zhu, H.~Lin, Y.~Lu, Y.~Lin, and S.~Han, ``Delayed gradient averaging:
  Tolerate the communication latency for federated learning,'' in \emph{Proc.
  Advances in Neural Information Processing Systems}, vol.~34, 2021.

\bibitem{2014delayed}
H.~R. Feyzmahdavian, A.~Aytekin, and M.~Johansson, ``A delayed proximal
  gradient method with linear convergence rate,'' in \emph{Proc. IEEE
  International Workshop on Machine Learning for Signal Processing}, 2014, pp.
  1--6.

\bibitem{GAN}
I.~Goodfellow, J.~Pouget-Abadie, M.~Mirza, B.~Xu, D.~Warde-Farley, S.~Ozair,
  A.~Courville, and Y.~Bengio, ``Generative adversarial nets,'' in \emph{Proc.
  Advances in Neural Information Processing Systems}, vol.~27, 2014.

\bibitem{KLD}
S.~Kullback and R.~A. Leibler, ``On information and sufficiency,'' \emph{The
  annals of mathematical statistics}, vol.~22, no.~1, pp. 79--86, 1951.

\bibitem{JSD}
J.~Lin, ``Divergence measures based on the shannon entropy,'' \emph{IEEE
  Transactions on Information theory}, vol.~37, no.~1, pp. 145--151, Jan. 1991.

\bibitem{WaD}
I.~Gulrajani, F.~Ahmed, M.~Arjovsky, V.~Dumoulin, and A.~C. Courville,
  ``Improved training of wasserstein gans,'' in \emph{Proc. Advances in Neural
  Information Processing Systems}, vol.~30, 2017.

\bibitem{WGAN}
M.~Arjovsky, S.~Chintala, and L.~Bottou, ``Wasserstein generative adversarial
  networks,'' in \emph{Proc. International Conference on Machine Learning},
  2017, pp. 214--223.

\bibitem{cifar}
A.~Krizhevsky and G.~Hinton, ``Learning multiple layers of features from tiny
  images,'' 2009.

\bibitem{mobilenetv3}
A.~Howard, M.~Sandler, G.~Chu, L.-C. Chen, B.~Chen, M.~Tan, W.~Wang, Y.~Zhu,
  R.~Pang, V.~Vasudevan \emph{et~al.}, ``Searching for mobilenetv3,'' in
  \emph{Proc. IEEE/CVF International Conference on Computer Vision}, 2019, pp.
  1314--1324.

\bibitem{adam}
D.~P. Kingma and J.~Ba, ``Adam: A method for stochastic optimization,'' in
  \emph{Proc. International Conference on Learning Representations}, 2015.

\bibitem{deb1995simulated}
K.~Deb, R.~B. Agrawal \emph{et~al.}, ``Simulated binary crossover for
  continuous search space,'' \emph{Complex systems}, vol.~9, no.~2, pp.
  115--148, 1995.

\bibitem{deb1996combined}
K.~Deb, M.~Goyal \emph{et~al.}, ``A combined genetic adaptive search (geneas)
  for engineering design,'' \emph{Computer Science and informatics}, vol.~26,
  no.~4, pp. 30--45, 1996.

\bibitem{while2006faster}
L.~While, P.~Hingston, L.~Barone, and S.~Huband, ``A faster algorithm for
  calculating hypervolume,'' \emph{IEEE transactions on evolutionary
  computation}, vol.~10, no.~1, pp. 29--38, 2006.

\end{thebibliography}

\end{document}